\documentclass{article}
\usepackage{arxiv}

\usepackage[utf8]{inputenc} 
\usepackage[T1]{fontenc}    
\usepackage[square,numbers]{natbib}
\usepackage{hyperref}       
\usepackage{url}            
\usepackage{booktabs}       
\usepackage{amsmath}
\usepackage{amsfonts}       
\usepackage{nicefrac}       
\usepackage{microtype}      

\usepackage{lipsum}         
\usepackage{graphicx}

\usepackage{doi}

\usepackage[table,xcdraw]{xcolor}
\usepackage[english]{babel}

\usepackage{cleveref}       

\usepackage{float}
\usepackage{amsthm}
\usepackage{mathrsfs}
\usepackage{amssymb}
\usepackage{pgfplots}
\usepackage{booktabs}
\usepackage{hhline}

\usepackage{doi}
\usetikzlibrary{spy}
\usepackage{tikz,pgfplots,pgf}
\usetikzlibrary{arrows,shapes,calc}
\usetikzlibrary{positioning}
\usepackage{neuralnetwork}
\usepackage[normalem]{ulem}

\newtheorem{example}{Example}
\newtheorem{theorem}{Theorem}
\newtheorem{proposition}{Proposition}
\newtheorem{definition}{Definition}

\newtheorem{corollary}{Corollary}
\newtheorem{lemma}{Lemma}
\newtheorem{notation}{Notation}
\newtheorem{remark}{Remark}
\newtheorem{method}{Method}
\pgfplotsset{compat=1.13}

\title{Max-min Learning of  Approximate Weight Matrices From Fuzzy Data}

\hypersetup{
pdftitle={Max-min Learning of  Approximate Weight Matrices From Fuzzy Data},
pdfsubject={cs.AI},
pdfauthor={Ismaïl Baaj},
pdfkeywords={Fuzzy set theory, Systems of fuzzy relational equations,  Learning, Possibility theory},
}
\author{Ismaïl Baaj \\ Univ. Artois, CNRS, CRIL, F-62300 Lens, France \\  \href{baaj@cril.fr}{baaj@cril.fr}}
\date{}

\begin{document}

\maketitle

\begin{abstract}
In this article, we study the approximate solutions set $\Lambda_b$ of an  inconsistent system of $\max-\min$ fuzzy relational equations $(S): A \Box_{\min}^{\max}x =b$. Using the $L_\infty$ norm, we compute by an explicit analytical formula the Chebyshev distance $\Delta~=~\inf_{c \in \mathcal{C}} \Vert b -c \Vert$, where $\mathcal{C}$ is the set of second members of the consistent systems defined with the same matrix $A$. We study 
the set $\mathcal{C}_b$ of Chebyshev approximations of the second member $b$ i.e., vectors $c \in \mathcal{C}$ such that $\Vert b -c \Vert = \Delta$, which is associated to the  approximate solutions set  $\Lambda_b$ in the following sense: an element of the set $\Lambda_b$ is a solution vector  $x^\ast$ of a system $A \Box_{\min}^{\max}x =c$ where $c \in \mathcal{C}_b$. As main results, we describe both the structure of the set $\Lambda_b$ and that of the set $\mathcal{C}_b$. 

We then introduce a paradigm for $\max-\min$ learning weight matrices  that relates input and output data from training data. The learning error is expressed in terms of the $L_\infty$ norm. We compute by an explicit formula the \textit{minimal value} of the learning error according to the training data. We give a method to construct weight matrices whose learning error is minimal, that we call approximate weight matrices. 

Finally, as an application of our results, 
we show how to learn approximately the rule parameters of a possibilistic rule-based system according to multiple training data.
\end{abstract}

\keywords{Fuzzy set theory ; Systems of fuzzy relational equations ;  Learning ; Possibility theory}

\section{Introduction}

Fuzzy relations were introduced in Zadeh's seminal paper on Fuzzy set theory \cite{zadeh1995fuzzy}.  The importance of fuzzy relations  was stressed  by Zadeh and Desoer in \cite{zadehlinear},  where they highlighted that the study of relations is equivalent to the study of systems, since a system can be viewed as relations between an input space and an output space. This perspective is emphasized in the foreword of \cite{di2013fuzzy}, where Zadeh wrote: ``human knowledge may be viewed as a collection of facts and rules, each of which may be represented as the assignment of a fuzzy relation to the unconditional or conditional possibility distribution of a variable. What this implies is that knowledge may be viewed as a system of fuzzy relational equations. In this perspective, then, inference from a body of knowledge reduces to the solution of a system of fuzzy relational equations''.

\noindent Thanks to Sanchez's pioneering work on solving a system of fuzzy $\max-\min$ relational equations \cite{sanchez1976resolution}, many Artificial Intelligence (AI) applications  based on systems of fuzzy relational equations have emerged \cite{baets2000analytical,di1991fuzzy,di2013fuzzy,dubois1995fuzzy,pedrycz1985applications}.
Sanchez gave necessary and sufficient conditions for a system of $\max-\min$ fuzzy relational equations to be consistent  i.e., to have solutions. In \cite{sanchez1977}, he also showed that, if the system is consistent, there is a greater solution and many minimal solutions, which leads him to describe  the complete set of solutions. 

However, addressing the inconsistency of these systems remains a difficult problem, which has often been raised \cite{baets2000analytical,di2013fuzzy,li2009fuzzy,pedrycz1990inverse}. 
Many authors have tackled the issue of finding approximate solutions   \cite{ cuninghame1995residuation,di2013fuzzy,gottwald1986characterizations,klir1994approximate, li2010chebyshev,pedrycz1990inverse,  wangming1986fuzzy,WEN2022,wu2022analytical,xiao2019linear,yager1980lack}, and some numerical approaches were presented \cite{pedrycz1983numerical,pedrycz1990algorithms}. Among these works, one pioneer idea  was introduced  by Pedrycz in \cite{pedrycz1990inverse}. Given an inconsistent system, Pedrycz proposes to slightly modify its second member to obtain a consistent system.  Cuninghame-Green and Cechlárová \cite{cuninghame1995residuation} and later Li and Fang \cite{li2010chebyshev}  each proposed an algorithm to measure the minimal distance   expressed with the $L_\infty$ norm $\Delta~=~\inf_{c \in \mathcal{C}} \Vert b -c \Vert$,  where $b$ is the second member  of a considered inconsistent system and $\mathcal{C}$ is the set of the  second members of the consistent systems defined with the same matrix: that of the inconsistent system. This minimal distance is called the Chebyshev distance associated to the second member of the inconsistent system. 

\noindent In this article, the first main result of our work is an explicit analytical formula (Theorem \ref{th:Deltamin}), to compute, for a system whose matrix and second member are respectively denoted $A$ and $b$,  the Chebyshev distance associated to its second member $b$. The Chebyshev distance is denoted $\Delta = \Delta(A,b)$ and is obtained by \textit{elementary calculations} involving only the components of the matrix $A$ and those of the second member $b$. Then, we tackle the study of the set $\mathcal{C}_b$ of Chebyshev approximations of the second member of the system, where a Chebyshev approximation is a vector $c$ such that $\Vert b -c \Vert = \Delta$ and the system formed by the matrix $A$ and the vector $c$ as second member is a consistent system. Moreover, we define the approximate solutions set $\Lambda_b$ of the system, and we relate $\Lambda_b$ to $\mathcal{C}_b$ in the following sense: an element of $\Lambda_b$ is a solution vector $x^\ast$ of a system whose matrix is $A$ and its second member is a Chebyshev approximation of $b$.

\noindent Motivated by Sanchez's seminal results \cite{sanchez1976resolution},  we introduce an idempotent application denoted $F$, see (\ref{eq:F}), to check if a system defined with a fixed matrix and a given vector used as second member is a consistent system. The properties of $F$ allow us to compute the greatest element of each of the sets  $\mathcal{C}_b$ and $\Lambda_b$, see (Proposition \ref{proposition:greatestcheb}) and (Proposition \ref{prop:greatestaprox}), i.e., we compute the greatest Chebyshev approximation of $b$ and the greatest approximate solution of the system from the components of the matrix $A$ and those of the vector $b$. Then, in order to give the structure of  the set $\mathcal{C}_b$, we study its minimal elements.  For this purpose, we give a first characterization of the set $\Lambda_b$ (Proposition \ref{proposition:firstlambdab}), which involves a system of $\max-\min$ inequalities. By relying on the results of  \cite{matusiewicz2013increasing}, we give a method for constructing the set $\mathcal{C}_{b,\min}$ of minimal Chebyshev approximations (Corollary \ref{corollary:ctitlde}) and we prove that it is non-empty and finite (Corollary \ref{cor:nonemptyfinite}). The complete structure of the set $\mathcal{C}_b$ follows from this results, see (Theorem \ref{th:3}). In addition, we prove a structure theorem for  the set $\Lambda_b$, see (Theorem \ref{th:2}).

 \noindent All these results let us  introduce a paradigm to learn  approximately a weight matrix relating input and output data from  training data. To our knowledge,  the $\max-\min$ learning of  a weight matrix is commonly tackled by trying to adapt the classical gradient descent method to  $\max-\min$ fuzzy neural networks \cite{blanco1994solving,blanco1995identification,blanco1995improved,ciaramella2006fuzzy,de1993neuron,hirota1982fuzzy,hirota1996solving,hirota1999specificity,ikoma1993estimation,li2017convergent,pedrycz1983numerical,pedrycz1991neurocomputations,pedrycz1995genetic,saito1991learning,stamou2000neural,teow1997effective,zhang1996min} with the aim of minimizing the learning error, which is expressed in terms of $L_2$ norm. However, the non-differentiability of the functions $\max$ and $\min$ is very challenging  for developing an efficient gradient descent method for $\max-\min$ neural networks. In our learning paradigm, we choose to express the learning error
 in terms of the $L\infty$ norm. We give an \textit{explicit formula for computing the minimal value (denoted by $\mu$) of the learning error according to the training data}, see (Definition \ref{def:mu}) and (Corollary \ref{eq:muegalmin}). The value $\mu$  is computed in terms of Chebyshev distance of the second member of systems of $\max-\min$ fuzzy relational equations associated to the training data.
We then give a method (Method \ref{method:defou}) for constructing approximate weight matrices, i.e., \textit{matrices whose learning error is equal to $\mu$}. 
Finally, we introduce analogous tools for a system of $\min-\max$
fuzzy relational equations to those already introduced for a system of  $\max-\min$ fuzzy relational equations and we show their correspondences (Table \ref{tab:correspondences}). This allows us to extend our results  in \cite{baaj2022learning}, i.e., we give a method for approximately learn rule parameters of a  possibilistic rule-based system according to multiple training data.

The article is structured as follows. In (Section \ref{sec:background}), we remind necessary and sufficient conditions for a system of $\max-\min$ fuzzy relational equations to be consistent.  We introduce the application $F$ and we give some of its useful properties. In (Section \ref{sec:chebyshev}), we give the explicit analytical formula for computing the Chebyshev distance associated to the second member of a system. In (Section \ref{sec:chebyshevapprox}), we define the set of Chebyshev approximations of the second member and compute the greatest Chebyshev approximation. In (Section \ref{sec:relating}), we describe the structure of the set of Chebyshev approximations and that of the approximation solutions set of the system. In (Section \ref{sec:learning}), we introduce our learning paradigm. In (Section \ref{sec:applications}), we show the correspondences between a system of $\min-\max$ fuzzy relational equations and a system of $\max-\min$ fuzzy relational equations and we present our method for  approximately learn rule parameters of a  possibilistic rule-based system according to multiple training data. Finally, we conclude with some perspectives.

\section{Background}
\label{sec:background}
In this section, we give the necessary background for solving  a system of $\max-\min$ fuzzy relational equations. We remind  Sanchez's necessary and sufficient condition for a system of $\max-\min$ fuzzy relational equations to be consistent. We reformulate this result as a fixed point property of a certain idempotent and increasing application, which we explicitly define. We show some of its useful properties.

\subsection{Solving of a system of \texorpdfstring{$\max-\min$}{max-min} fuzzy relational equations}
 
\noindent We use the following notation:
\begin{notation}  $[0,1]^{n\times m}$ denotes the set of matrices of size $(n,m)$ i.e., $n$ rows and $m$ columns, whose components are in $[0,1]$. In particular:
\begin{itemize}
    \item $[0,1]^{n\times 1}$ denotes the set of column vectors of $n$ components,
    \item $[0,1]^{1\times m}$ denotes  the set of row matrices of $m$ components. 
\end{itemize}

\noindent In the set $[0,1]^{n\times m}$, we use the order relation $\leq$ defined by:
\[ A \leq B \quad \text{iff we have} \quad  a_{ij} \leq b_{ij} \quad \text{ for all } \quad 1 \leq i \leq n, 1 \leq j \leq m,    \]
\noindent where $A=[a_{ij}]_{1 \leq i \leq n, 1 \leq j \leq m}$ and $B=[b_{ij}]_{1 \leq i \leq n, 1 \leq j \leq m}$. \\ 
\end{notation}

\noindent Let $A=[a_{ij}] \in [0,1]^{n\times m}$ be a matrix of size $(n,m)$ and $b=[b_{i}] \in [0,1]^{n\times 1}$ be a vector of $n$ components. The system of $\max-\min$ fuzzy relational equations associated to  $(A,b)$ is of the form:
\begin{equation}\label{eq:sys}
    (S): A \Box_{\min}^{\max} x = b,
\end{equation}
\noindent where $x = [x_j]_{1 \leq j \leq m} \in [0,1]^{m\times 1}$ is  an unknown vector of $m$ components and the operator $\Box_{\min}^{\max}$ is the matrix product that uses the t-norm $\min$ as the product and $\max$  as the addition. The system can also be written as:
\begin{equation*}
    \max_{1 \leq j \leq m} \min(a_{ij},x_j) = b_i,\, \forall i \in \{1, 2,\dots, n\}.
\end{equation*}
\noindent There are two competing notation conventions for studying systems of fuzzy relational equations: they differ in whether the unknown part and the second member are column vectors or row vectors. These two conventions are  equivalent and the transpose  map allows us to switch from one to the other.

\noindent To check if the system $(S)$ is consistent, we compute the following vector: \begin{equation}\label{eq:egretestsol}
e = A^t \Box_{\rightarrow_G}^{\min} b,\end{equation}
\noindent where $A^t$ is the transpose of $A$ and the matrix product $\Box_{\rightarrow_G}^{\min}$ uses the Gödel implication $\rightarrow_G$ as the product and $\min$ as the addition. The Gödel implication is defined by:
\begin{align}\label{eq:implication:godel}
     x \rightarrow_G y = \begin{cases}1 & \text{ if } x \leq y \\ y &\text{ if } x > y \end{cases}.
\end{align}
Thanks to Sanchez's seminal work \cite{sanchez1976resolution}, we have the following equivalence:
\begin{equation}
    (S) \text{ is consistent}\Longleftrightarrow A \Box_{\min}^{\max}  e = b. 
\end{equation}
\noindent  The set of solutions of the system $(S)$ is denoted by:
    \begin{equation}\label{eq:setofsolutionsS}
{\cal  S  } = {\cal  S  }(A , b) = \{ v \in [0 , 1]^{m \times 1} \,\mid \,  A \, \Box_{\min}^{\max} v  = b \}. 
\end{equation}
\noindent If the system $(S)$ is consistent, the vector $e$, see (\ref{eq:egretestsol}), is the greatest solution of the system $(S)$. Sanchez also showed in \cite{sanchez1977} that the system $(S)$ has many minimal solutions and he described its set $\mathcal{S}$ of solutions.\\ 
\noindent We begin our study by the following useful result:
\begin{lemma}\label{lemma:increasing}
The maps:
\begin{equation}\label{eq:map:amaxminx}
    [0,1]^{m\times 1}\rightarrow [0,1]^{n \times 1} : x \mapsto  A \Box_{\min}^{\max} x,
\end{equation}
\begin{equation}\label{eq:map:atmingc}
[0,1]^{n \times 1} \rightarrow [0,1]^{m\times 1} : c \mapsto A^t \Box_{\rightarrow_G}^{\min} c
\end{equation}
\noindent are increasing with respect to the usual order  relation  between vectors. 
\end{lemma}
\begin{proof}
The first map is increasing because the $\max$ and $\min$ functions are increasing. For the second map, one can use that for a fixed  $x \in [0 , 1]^{m \times 1}$ , the map $y \mapsto (x \rightarrow_G y)$ is increasing. 
\end{proof}
\noindent As a consequence, we have the following well-known result:
\begin{lemma}\label{lemma:appF}
Let $c,c' \in {[0,1]}^{n \times 1}$ such that $c \leq c'$ then we have: 
\begin{equation}
    \forall v \in [0,1]^{m \times 1}, A \Box_{\min}^{\max} v = c  \Longrightarrow v \leq A^t \Box_{\rightarrow_G}^{\min} c'.
\end{equation}
\end{lemma}

\begin{proof}
\noindent Let us remind that $e = A^t \Box_{\rightarrow_G}^{\min} c$ and $e' = A^t \Box_{\rightarrow_G}^{\min} c'$ are the potential greatest solutions of the systems $A \Box_{\min}^{\max} x = c$ and $A \Box_{\min}^{\max} x =c'$ respectively. Then,  from  (\ref{eq:map:atmingc})  we deduce $e \leq e'$.\\
Let $v \in [0,1]^{m \times 1}$ be such that $A \Box_{\min}^{\max} v = c$. Then the system $A \Box_{\min}^{\max} x = c$ is consistent and $v \leq e$. By transitivity of the relation order, we get $v \leq e'$.
\end{proof}

\noindent We illustrate the solving of the system $(S)$ by an example:

\begin{example}
Let:
\begin{equation*}
    A = \begin{bmatrix}
        0.06 & 0.87 & 0.95\\
        0.75 & 0.13 & 0.88\\
        0.82 & 0.06 & 0.19
    \end{bmatrix}\text{ and }
    b = \begin{bmatrix}
    0.4\\0.7\\ 0.7
    \end{bmatrix}.
\end{equation*}
We have: $A^t = \begin{bmatrix}0.06& 0.75& 0.82\\ 0.87 &0.13& 0.06\\
    0.95 & 0.88 & 0.19
    \end{bmatrix}$. We compute the potential greatest solution:
\begin{equation*}
    e = A^t \Box_{\rightarrow_G}^{\min} b = \begin{bmatrix}\min(1.0,0.7,0.7)\\ \min(0.4,1.0,1.0)\\ \min(0.4,0.7,1.0)\end{bmatrix}= \begin{bmatrix}
    0.7 \\
    0.4 \\
    0.4
    \end{bmatrix}.
\end{equation*}
The system $A \Box_{\min}^{\max} x = b$ is consistent because:
\[  A \Box_{\min}^{\max}  e = \begin{bmatrix}
    0.4\\0.7\\ 0.7
    \end{bmatrix}  = b.\]
\end{example}

\subsection{Reformulation of Sanchez's condition as a fixed point property}

\noindent For the system $(S)$, we introduce the following application:
\begin{equation}\label{eq:F}
    F : [0,1]^{n \times 1} \rightarrow [0,1]^{n \times 1} :
 c \mapsto F(c) = A  \, \Box_{\min}^{\max} (A^t\, \Box_{\rightarrow_G}^{\min} c).
\end{equation}

\noindent The application $F$ allows us to check if a system of fuzzy relational equations $\max-\min$ is consistent:

\begin{proposition}\label{proposition:appFreformulated}
For any  vector $c \in [0,1]^{n \times 1}$ the following conditions are equivalent:
\begin{enumerate}
    
    \item $F(c) = c$, 
    \item the system $ A  \Box_{\min}^{\max} x = c$ is consistent.
\end{enumerate}
 \end{proposition}
\begin{proof}
    Reformulation of Sanchez's result. 
\end{proof}

\noindent The properties of idempotence, growth and right-continuity of the application $F$ justify its introduction: 
\begin{proposition}\label{prop:threestatementF}
\mbox{}
\begin{enumerate}
\item $\forall c \in [0,1]^{n \times 1}$, $F(c) \leq c$.
    \item $F$ is idempotent i.e., $\forall c \in [0,1]^{n \times 1}, F(F(c)) = F(c)$. 

    \item  $F$ is increasing and right-continuous.
\end{enumerate}
\end{proposition}
\noindent The application $F$ being right-continuous  at a point $c\in [0,1]^{n \times 1}$ means: for any sequence $(c^{(k)})$  in  $[0,1]^{n \times 1}$ such that  $(c^{(k)})$ converges to $c$ when  $k \rightarrow \infty$ and verifying $\forall k, c^{(k)} \geq c$, we have:  
\begin{center}
    $F(c^{(k)}) \rightarrow F(c)$ when $k \rightarrow \infty$.
\end{center}
\begin{proof}\mbox{}
\begin{enumerate}
    \item Let $i \in \{1,2,\dots,n\}$, we denote by ${F(c)}_i$ (resp. $c_i$) is the $i$-th component of the vector ${F(c)}$ (resp. $c$) and  we must prove ${F(c)}_i \leq c_i$. We have:
    \begin{align}
        {F(c)}_i &= \max_{1 \leq j \leq m} \min[ a_{ij}, \min_{1 \leq k \leq n} a_{kj} \rightarrow_G c_k ] \nonumber\\
        & \leq \max_{1 \leq j \leq m} \min[ a_{ij},  a_{ij} \rightarrow_G c_i ]
        \nonumber\\
        &= \max_{1 \leq j \leq m} \min[ a_{ij},  c_i ] \quad (\text{because } \min(x,x\rightarrow_G y) = \min(x,y)) \nonumber\\
        &\leq c_i.\nonumber
    \end{align}
    \item Consider the system $A \Box_{\min}^{\max} x = F(c)$. \\ By definition of the application $F$, we have $F(c) = A  \, \Box_{\min}^{\max} e$ with $e = A^t \Box_{\rightarrow_G}^{\min} c$. By Sanchez's result, we have: $$e \leq A^t  \Box_{\rightarrow_G}^{\min} F(c).$$
    From (\ref{eq:map:amaxminx}) we get:
    \begin{equation*}
        F(c) =  A  \Box_{\min}^{\max} e \leq A \Box_{\min}^{\max} (A^t  \Box_{\rightarrow_G}^{\min} F(c)) =F(F(c)).
    \end{equation*}
    \noindent But from the first statement of (Proposition \ref{prop:threestatementF}), we know that $F(F(c)) \leq F(c)$. Therefore, we have $F(F(c))=F(c)$. 
    \item 
    This follows from the fact that for a fixed $x \in [0 , 1]$, the map $y \mapsto (x \rightarrow_G y)$ is right-continuous.
\end{enumerate}
\end{proof}
\noindent We illustrate the use of the application $F$:
\begin{example}
(continued) 
Based on the computations in the previous example, we check that $F(b) = b$. 
Let $c = \begin{bmatrix}0.36\\ 0.57\\ 0.24\end{bmatrix}$. The potential greatest solution of the system $A \Box_{\min}^{\max} x = c$ is $\begin{bmatrix}0.24\\ 0.36\\ 0.36\end{bmatrix}$. We have $F(c) = \begin{bmatrix}0.36\\ 0.36\\ 0.24\end{bmatrix} \neq c$, so the system $A \Box_{\min}^{\max} x = c$  is not consistent. 
\end{example}

\section{Chebyshev distance associated to the second member of the system \texorpdfstring{$(S)$}{(S)}}
\label{sec:chebyshev}
In this section, we give an analytical method for computing the Chebyshev distance associated to the second member of the system $(S)$, see (\ref{eq:sys}).  For this purpose, we begin by giving some notations and studying two inequalities involved in the computation of this Chebyshev distance. We relate the fundamental result (Theorem 1 of \cite{cuninghame1995residuation}) to the properties of the studied inequalities. 
 This allows us to give an explicit formula for computing the Chebyshev distance associated to the second member $b$ of the system $(S)$. 

\subsection{Notations}
For $x,y,z,u,\delta \in [0,1]$, we use the following notations:
\begin{itemize}
    \item $x^+ = \max(x,0)$,
    \item $\overline{z}(\delta) = \min(z+\delta,1)$, 
    \item $\underline{z}(\delta) = \max(z-\delta,0) = (z-\delta)^+$.
\end{itemize}

\noindent We remark that we have the following equivalence in $[0,1]$: 
\begin{equation}\label{ineq:xyxbarybar}
    \mid x - y \mid \leq \delta \Longleftrightarrow \underline{x}(\delta) \leq y \leq \overline{x}(\delta).
\end{equation}

\noindent For our work, to the second member $b = [b_i]_{1 \leq  i \leq n}$ of the system $(S)$ and a number $\delta \in [0,1]$, we associate two vectors:
    \begin{equation}\label{def:bstarhautbas} 
  \underline{b}(\delta) = [(b_i  - \delta)^+]_{1 \leq  i \leq n} \quad \text{and} \quad \overline{b}(\delta) = [\min(b_i  + \delta , 1)]_{1 \leq  i \leq n}.\end{equation}
\noindent These vectors $\underline{b}(\delta)$ and $\overline{b}(\delta)$ were already introduced e.g.,  in  \cite{cuninghame1995residuation} (with others notations) and in \cite{li2010chebyshev}. \\

\noindent Then, from (\ref{ineq:xyxbarybar}), we deduce for any $c = [c_i]_{1\leq i \leq n} \in [0,1]^{n \times 1}$: 
\begin{equation}\label{ineq:bbbar}
    \Vert b - c \Vert \leq \delta \Longleftrightarrow \underline{b}(\delta) \leq c \leq \overline{b}(\delta).
\end{equation}
\noindent where $ \Vert b - c \Vert = \max_{1 \leq i \leq n}\mid b_i-c_i\mid$.

\subsection{Preliminaries}
\noindent  Let $x,y,z,u \in [0,1]$ be fixed. We study how to obtain the smallest value of $\delta \in [0,1]$ so that the following inequality is true:
\begin{equation*}
  (x-\delta)^+ \leq y.
\end{equation*}

\begin{proposition}\label{proposition:ineq1}
For any $\delta\in [0, 1]$, we have: \begin{equation}\label{ineq:gen:first}(x-\delta)^+ \leq y\Longleftrightarrow (x - y)^+ \leq \delta.\end{equation}
\end{proposition}
\begin{proof}
For any $\delta\in [0, 1]$, we have:
\begin{align}
 (x-\delta)^+ \leq y & \Longleftrightarrow  \max(x - \delta, 0) - y \leq 0 \nonumber\\
 &   \Longleftrightarrow \max(x - y - \delta, - y)   \leq 0 \nonumber\\
  &   \Longleftrightarrow x - y - \delta     \leq 0 \nonumber\\
  &   \Longleftrightarrow x - y \leq \delta        \nonumber\\
  &   \Longleftrightarrow (x - y)^+ \leq \delta.        \nonumber
\end{align}    
\end{proof}
\noindent
We deduce from  (Proposition \ref{proposition:ineq1}) that we have: 
\begin{equation}\label{ineq:P1}
    (x - y)^+ = \min\{\delta\in [0, 1] \,\mid\, (x-\delta)^+ \leq y\}.
\end{equation} 
\noindent Let us study the solving of the following inequality that involves the Gödel implication $\rightarrow_{G}$, see (\ref{eq:implication:godel}): \begin{equation}
    \underline{x}(\delta) \leq y \rightarrow_{G} \overline{z}(\delta), 
\end{equation}
\noindent where: $
    y \rightarrow_{G} \overline{z}(\delta) = \begin{cases}1 & \text{ if } y - z \leq \delta \\
z+ \delta & \text{ if } y -z > \delta\end{cases}.$\\
\noindent Let:
\begin{equation}
    \sigma_G(x, y, z) = \min( \frac{(x - z)^+}{2},(y - z)^+).
\end{equation}

\noindent Then:
\begin{proposition}\label{proposition:ineq2}
For any $\delta\in [0, 1]$,  we have:
\begin{equation}
    \label{ineq:gen:tm}
    \underline{x}(\delta) \leq y \rightarrow_{G} \overline{z}(\delta)
    \Longleftrightarrow \sigma_G(x, y, z) \leq \delta.
\end{equation}
\end{proposition}
\begin{proof}\mbox{}\\
\noindent $\Longrightarrow$ Let us assume    $ \underline{x}(\delta) \leq  y  \rightarrow_G \overline z(\delta)$ and  prove $\sigma_G(x,y,z) \leq \delta$.

\noindent
We remark that:
\begin{itemize}
    \item  If $y \leq z$ or $x \leq z$, then  $\sigma_G(x, y, z) = 0  \leq \delta$. It remains for us to study the case where $y > z$ and  $x > z$.
    \item If $y - z \leq  \delta $, then $\sigma_G(x, y, z)  =  \min(\frac{ x - z }{2}, y -  z) \leq y - z \leq  \delta $. \\
It remains for us to study the case where $y - z > \delta$.
\end{itemize}

\noindent
We have:
\begin{align}\underline{x}(\delta) -   (y  \rightarrow_G \overline z(\delta)) & = \max(x - \delta, 0) - z - \delta \nonumber\\
&= \max(x - \delta  - z - \delta, - z - \delta)   \nonumber\\
&= \max(x - z - 2\delta  , - z - \delta) \leq 0, \nonumber\end{align}
so  $x - z - 2\delta  \leq 0$  and $\sigma_G(x, y, z)  =  \min(\frac{ x - z }{2}, y -  z) \leq \frac{ x - z }{2} \leq \delta$.\\

\noindent $\Longleftarrow$ Let us assume  $ \sigma_G(x, y, z)  \leq \delta$ and prove $ \underline{x}(\delta) \leq  y  \rightarrow_G \overline z(\delta)$. \\
If $(x - \delta)^+ = 0$ or 
$y  \rightarrow_G \overline z(\delta) = 1$, we  trivially get the inequality $ (x - \delta)^+ \leq  y  \rightarrow_G \overline z(\delta)$. It remains for us to study the case where $(x - \delta)^+  = x - \delta > 0$ and
$y  \rightarrow_G \overline z(\delta) < 1$.

\noindent
From the inequality  $y  \rightarrow_G \overline z(\delta) < 1$, we deduce:
$$y - z > \delta \text{ and }  y  \rightarrow_G \overline z(\delta) = z + \delta.$$
As   $\sigma_G(x, y, z) =   \min(\dfrac{(x - z)^+}{2}, (y -  z)^+) =   \min(\dfrac{(x - z)^+}{2}, y -  z) \leq \delta$, we obtain: 
$$\sigma_G(x, y, z)=\dfrac{(x - z)^+}{2} \leq \delta.$$
This last inequality  is equivalent to:
$$\max(x - z  - 2\delta, -2\delta) \leq 0.$$
So $x - z  - 2\delta \leq 0$, which implies: 
$$(x - \delta)^+ =  x - \delta  \leq z + \delta  = y  \rightarrow_G \overline z(\delta).$$ 
\end{proof}
\noindent
We deduce from  (Proposition \ref{proposition:ineq2}) that we have:
\begin{equation}\label{ineq:P2}
    \sigma_G(x, y, z) =    \min\{\delta\in [0, 1] \,\mid\, \underline{x}(\delta) \leq  y  \rightarrow_G \overline z(\delta)\}.
\end{equation}
\noindent We illustrate this result:
\begin{example}
Let $x = 0.56,  y = 0.87$ and $z = 0.36$. We want  to obtain the smallest value of $\delta \in [0,1]$ so that $ \underline{x}(\delta) \leq y \rightarrow_{G} \overline{z}(\delta)$ is true. We have $y \rightarrow_G z = z$ and $x > z$.
\begin{align*}
    \delta &\!\begin{aligned}[t]
    &= \sigma_G(x, y, z) \\
     &=\min( \frac{(x - z)^+}{2},(y - z)^+)\\
     &=\min( \frac{(0.56 - 0.36)^+}{2},(0.87 - 0.36)^+)\\
     &=\min( \frac{0.20}{2}, 0.51) \\
     &= 0.10.
     \end{aligned}
\end{align*}
\noindent We have $\underline{x}(\delta)=x - 0.10 = 0.46$ and $\overline{z}(\delta)=z+0.10 = 0.46$. Therefore:
\begin{equation*}
   \underline{x}(\delta) = y \rightarrow_G \overline{z}(\delta).
\end{equation*}
\end{example}

\subsection{Analytical formula for computing the Chebyshev distance associated to the second member of the system \texorpdfstring{$(S)$}{(S)}} 
\label{sec:tool:cheb}

\noindent To the matrix $A$ and the vector $b$ of the system $(S)$, let us  associate  the set of  vectors $c = [c_i] \in [0,1]^{n \times 1}$ such that the system $A \Box_{\min}^{\max} x = c$ is consistent:
\begin{equation}\label{def:setofsecondmembersB}
    \mathcal{C} = \{ c = [c_i] \in {[0,1]}^{n \times 1} \mid  A \Box_{\min}^{\max} x = c \text{ is consistent} \}.
\end{equation}
\noindent This set allows us to define the Chebyshev distance associated to the second member $b$ of the system $(S)$.
 \begin{definition}\label{def:chebyshevdist}
The Chebyshev distance associated to the second member $b$ of the system $(S): A \Box_{\min}^{\max}x = b$ is: 
\begin{equation}\label{eq:delta}
    \Delta = \Delta(A,b) =  \inf_{c \in \mathcal{C}} \Vert b - c \Vert 
    \end{equation}

    \end{definition}
\noindent where:
\[ \Vert b - c \Vert = \max_{1 \leq i \leq n}\mid b_i - c_i\mid.\]

We have the following fundamental result,  already proven in \cite{cuninghame1995residuation}: 

\begin{equation}\label{eq:delta1}
    \Delta =        \min\{\delta\in [0, 1] \mid \underline b(\delta) \leq F(\overline b(\delta))\}.
    \end{equation}

\noindent In the following, using only (\ref{eq:delta1}), we prove that the Chebyshev distance $\Delta$ associated to  the second member $b$ of the system $(S)$ is given by the following formula:
\begin{theorem}\label{th:Deltamin}
\begin{equation} \label{eq:Deltamin}
\Delta = \max_{1 \leq i \leq n}\,\delta_i
\end{equation}
where for $i = 1, 2, \dots n$: 
\begin{equation} \label{eq:Deltamini} 
\delta_i =  \min_{1 \leq j \leq m}\,\max[ (b_i - a_{ij})^+,  \max_{1 \leq k \leq n}\,  \,\sigma_G\,(b_i, a_{kj}, b_k)].
\end{equation}
\end{theorem}

\noindent To prove this formula, let us first introduce some notations and a lemma:
\begin{notation}
\end{notation}
\noindent  For $1\leq i,k \leq n$ and  $1 \leq j \leq m$ let:
\begin{itemize}
    \item $K_i =  \{\delta \in [0, 1] \mid  {\underline{b}(\delta)}_i \leq {F(\overline b(\delta))}_i \}$, where  ${\underline{b}(\delta)}_i$ (resp. ${F(\overline b(\delta))}_i$) is the $i$-th component of the vector ${\underline{b}(\delta)}$ (resp. ${F(\overline b(\delta))}$),
    \item $\beta_j  = \min_{1\leq   k \leq n}\, a_{kj} \rightarrow_G  {\overline{b}(\delta)}_k$ where  ${\overline{b}(\delta)}_k$  is the $k$-th component of the vector ${\overline{b}(\delta)}$,
    \item $D_{ij}^A = \{\delta \in [0, 1] \mid \underline b(\delta)_i \leq a_{ij}\}$,
    \item $D_{ij}^\beta = \{\delta \in [0, 1] \mid \underline b(\delta)_i \leq \beta_j\}$,
    \item $D_{ijk} = \{\delta \in [0, 1] \mid \underline b(\delta)_i \leq  a_{kj} \rightarrow_G \overline{b}(\delta)_k\}$. 
\end{itemize}

\begin{lemma}\label{L2}
We have: \begin{equation*}\label{eq:deltamin}K_i  =  \bigcup_{1\leq j \leq m}\, D_{ij}^A\,\cap \,D_{ij}^\beta 
\text{ and }
  D_{ij}^\beta = \bigcap_{1 \leq k \leq n}\, D_{ijk}.
\end{equation*}
\end{lemma}

\begin{proof}
By definition of the function $F$, we have: 
$$F(\overline b(\delta))_i  = 
\max_{1 \leq j \leq m}\, \min(a_{ij}, \beta_j).$$ 

\noindent This implies directly that we have: 
$$K_i = \bigcup_{1\leq j \leq m}\,  D_{ij}^A \cap D_{ij}^\beta.$$
 
\noindent  As $\beta_j = \min_{1 \leq k \leq n}\, a_{kj} \rightarrow_G \overline{b}(\delta)_k$, we also have: 
 $$ D_{ij}^\beta = \bigcap_{1 \leq k \leq n}\, D_{ijk}.
  $$
 
\end{proof}
\noindent The proof of (Theorem \ref{th:Deltamin}) is given in the following.

\begin{proof} For any $i = 1, 2, \dots, n$ and $j = 1, 2, \dots, m$, we deduce from (Proposition \ref{proposition:ineq1}) and (Proposition \ref{proposition:ineq2}) that for any 
$\delta\in [0, 1]$, we have: 
$$\delta \in D_{ij}^A  
\,\Longleftrightarrow\,
 \delta \geq (b_i - a_{ij})^+ \text{ and }
\delta \in D_{ij}^\beta  
\,\Longleftrightarrow\,
 \delta \geq \max_{1 \leq k \leq n}\,  \,\sigma_G\,(b_i, a_{kj}, b_k).$$
Using (\ref{eq:deltamin}), we get: 
$$\delta \in K_i 
\,\Longleftrightarrow\,
\exists \, j\in\{1, 2, \dots, m\}  \text{ such that } \delta \geq \max[ (b_i - a_{ij})^+,  \max_{1 \leq k \leq n}\,  \,\sigma_G\,(b_i, a_{kj}, b_k) ].
$$
So, we obtain: 
$$\delta \in K_i 
\,\Longleftrightarrow\, 
 \delta \geq \min_{1 \leq j \leq m}\, \max[ (b_i - a_{ij})^+,  \max_{1 \leq k \leq n}\,  \,\sigma_G\,(b_i, a_{kj}, b_k) ].$$ 
As by definition $\delta \in K_i    
\,\Longleftrightarrow\,
\underline b(\delta)_i \leq F(\overline b(\delta))_i$ and $\Delta =         \min\{\delta\in [0, 1] \mid \underline b(\delta) \leq F(\overline b(\delta))\}$, see (\ref{eq:delta1}), we get: 
$$\Delta = \max_{1 \leq i \leq n}\, \min_{1 \leq j \leq m}\, \max[ (b_i - a_{ij})^+,  \max_{1 \leq k \leq n}\,  \,\sigma_G\,(b_i, a_{kj}, b_k) ].$$
\end{proof}
\noindent The following example illustrates the computation of the Chebyshev distance associated to the second member of the system $(S)$:
\begin{example}\label{ex:number4}
Let: 
\begin{equation}\label{ex:4:Aandb}
    A = 
        \begin{bmatrix}
0.03&0.38&0.26\\
0.98&0.10&0.03\\
0.77&0.15&0.85\\
\end{bmatrix}
\text{ and } b = \begin{bmatrix}
0.54 \\
0.13 \\
0.87 \\
\end{bmatrix}.
\end{equation}
\noindent We apply (Theorem \ref{th:Deltamin}). 
\noindent We compute:
\begin{align*}
    \delta_1 &\!\begin{aligned}[t]
    &=  \min_{1 \leq j \leq 3}\,\max[ (b_1 - a_{1j})^+,  \max_{1 \leq k \leq 3}\,  \,\sigma_G\,(b_1, a_{kj}, b_k)].\\
     \end{aligned}
\end{align*}
We have: \begin{equation*}
    [(b_1 - a_{1j})^+]_{1 \leq j \leq 3} = \begin{bmatrix}
        0.54 - 0.03\\
        0.54 - 0.38\\
        0.54 - 0.26\\
    \end{bmatrix}= \begin{bmatrix}
        0.51\\
        0.16\\
        0.28
    \end{bmatrix},
\end{equation*}
\begin{align*}
    [\sigma_G\,(b_1, a_{kj}, b_k)]_{1 \leq k \leq 3,1 \leq j \leq 3} &=
    \begin{bmatrix}
    \sigma_G\,(b_1, a_{11}, b_1) & \sigma_G\,(b_1, a_{12}, b_1) & 
    \sigma_G\,(b_1, a_{13}, b_1) \\
    \sigma_G\,(b_1, a_{21}, b_2) & 
    \sigma_G\,(b_1, a_{22}, b_2) & 
    \sigma_G\,(b_1, a_{23}, b_2) \\
    \sigma_G\,(b_1, a_{31}, b_3) & 
    \sigma_G\,(b_1, a_{32}, b_3) &
    \sigma_G\,(b_1, a_{33}, b_3) \\
    \end{bmatrix} \\
    &=
    \begin{bmatrix}
    \sigma_G\,(0.54, 0.03, 0.54) & 
    \sigma_G\,(0.54, 0.38, 0.54) & 
    \sigma_G\,(0.54, 0.26, 0.54)\\
    \sigma_G\,(0.54, 0.98, 0.13) &
    \sigma_G\,(0.54, 0.10, 0.13) & 
    \sigma_G\,(0.54, 0.03, 0.13) \\
    \sigma_G\,(0.54, 0.77, 0.87) &
    \sigma_G\,(0.54, 0.15, 0.87) &
     \sigma_G\,(0.54, 0.85, 0.87) \\
    \end{bmatrix}\\
    &= \begin{bmatrix}
0.0&0.0&0.0\\
0.205&0.0&0.0\\
0.0&0.0&0.0\\
\end{bmatrix}.
\end{align*}
Therefore:
\begin{equation*}
    \delta_1 = \min(\max(0.51,0.205),\max(0.16,0),\max(0.28,0))= 0.16.
\end{equation*}
Similarly, we obtain: $\delta_2 = 0$ and $\delta_3 = 0.02$. 
\noindent The Chebyshev distance associated to $b$  is $\Delta = \max(\delta_1,\delta_2,\delta_3) = 0.16$.

\end{example}

\section{Chebyshev approximations  of the second member of the system \texorpdfstring{$(S)$}{(S)}}
\label{sec:chebyshevapprox}
\noindent In this section, we  study the  Chebyshev approximations of the second member $b$ of the system $(S)$, which are vectors $c \in [0,1]^{n \times 1}$ such that $\Vert b - c \Vert = \Delta$ and the system $A \Box_{\min}^{\max} x = c$ is consistent. We show that there is a greater Chebyshev approximation that we compute. We give the definition of the set of minimal Chebyshev approximations, which will be useful to determine the structure of the set of Chebyshev approximations.

We define the set of Chebyshev approximations of $b$:
\begin{definition}
   The set of Chebyshev approximations of $b$ is defined using the set $\mathcal{C}$, see (\ref{def:setofsecondmembersB}), and the Chebyshev distance associated to $b$ (Definition \ref{def:chebyshevdist}): 
\begin{equation}\label{def:chebyshevsetapproxB} 
{\cal C}_{b}  = \{c \in {\cal C} \,\mid \,  \Vert b - c\Vert = \Delta(A ,  b)\}.
\end{equation} 
\end{definition}

\noindent In the following, to prove that the set  ${\cal C}_{b}$ is non-empty,  we show that the greatest Chebyshev approximation exists, according to the usual order relation between vectors of $[0,1]^{n \times 1}$.

\noindent 
\begin{proposition}\label{proposition:greatestcheb}
\mbox{}
\begin{enumerate}
    \item $F(\overline{b}(\Delta)) \in \mathcal{C}_b$,
    \item $\forall c \in {\cal C}_{b},\,  c \leq F(\overline{b}(\Delta)).$
\end{enumerate}
So, $F(\overline{b}(\Delta))$ is the greatest Chebyshev approximation of $b$.
\end{proposition}
\begin{proof}
We deduce from (Proposition \ref{proposition:appFreformulated}) and (Proposition \ref{prop:threestatementF})  that $F(\overline{b}(\Delta)) \in \mathcal{C}$.
From (\ref{eq:delta1}) and (Proposition \ref{prop:threestatementF}) we deduce:
\[ \underline{b}(\Delta) \leq F(\overline{b}(\Delta)) \leq \overline{b}(\Delta).\]
\noindent From relation (\ref{ineq:xyxbarybar}),  we deduce $\Vert F(\overline{b}(\Delta)) - b\Vert \leq \Delta$. But $\Delta = \inf_{c \in \mathcal{C}} \Vert b - c \Vert$ (Definition  \ref{def:chebyshevdist}), then $\Vert F(\overline{b}(\Delta)) - b\Vert \geq \Delta$.
Finally, $\Vert F(\overline{b}(\Delta)) - b\Vert = \Delta$ i.e., $F(\overline{b}(\Delta)) \in \mathcal{C}_b$.\\

\noindent Let  $c$ be a vector in $\mathcal{C}_b$. As $\Vert b -c \Vert = \Delta$, we deduce $c \leq \overline{b}(\Delta)$. Using that $F$ is increasing (Proposition \ref{prop:threestatementF}), we have  $F(c) \leq F(\overline{b}(\Delta))$. But $F(c) =c$ (Proposition \ref{proposition:appFreformulated}), so $c \leq F(\overline{b}(\Delta))$.
\end{proof}
\noindent As a consequence of the first statement in (Proposition \ref{proposition:greatestcheb}), we have: 
\begin{corollary}
    \label{corollary:deltazero}
    $$\Delta = \min_{c \in \mathcal{C}} \Vert b - c \Vert.$$
   \[
       \Delta = 0 \Longleftrightarrow\text{ the system $(S)$ is consistent.} 
   \]
\end{corollary}
\noindent Therefore, $\Delta = 0$ is a  necessary and sufficient condition for the system $(S)$ to be consistent. 

\noindent It is much more difficult to obtain minimal Chebyshev approximations of $b$. In a fairly abstract way, one can prove that the set: 
\begin{equation}\label{eq:minimalChebB}
{\cal C}_{b,\min}  = \{c \in{\cal C}_{b} \mid   c \,\text{minimal in ${\cal C}_{b}$}\}
\end{equation}
\noindent is non-empty and finite. In fact, we will show in the next section how to construct elements of ${\cal C}_{b,\min}$ and prove  that this set is finite.

\noindent We illustrate the computation of the greatest Chebyshev approximation of the second member of the system $(S)$:
\begin{example}\label{ex:example5}
(continued) 
We continue with the matrix $A$ and the vector $b$, see (\ref{ex:4:Aandb}), used in (Example \ref{ex:number4}). \\We remind that the Chebyshev distance associated to the second member $b$ of the system $A \Box_{\min}^{\max} x = b$  is $\Delta = 0.16$. 

\noindent From $b = \begin{bmatrix}
0.54 \\
0.13 \\
0.87 \\
\end{bmatrix}$, we compute $\overline{b}(\Delta) = \begin{bmatrix}
        0.70\\
        0.29\\
        1.00
    \end{bmatrix}$. \\
\noindent Then, the greatest Chebyshev approximation of $b$ is:
\begin{equation*}
 F(\overline{b}(\Delta)) =  A  \, \Box_{\min}^{\max} (A^t\, \Box_{\rightarrow_G}^{\min} \overline{b}(\Delta)) =  \begin{bmatrix}
0.38 \\
0.29\\
0.85 \\
\end{bmatrix}.
\end{equation*}
We check that the distance between the greatest  Chebyshev approximation  $\begin{bmatrix}
0.38 \\
0.29\\
0.85 \\
\end{bmatrix}$  and  $b = \begin{bmatrix}
0.54 \\
0.13 \\
0.87 \\
\end{bmatrix}$ is equal to $\Delta$.
\end{example}

\section{Relating the approximate solutions set to the Chebyshev approximations set}
\label{sec:relating}

In this section, we study the  approximate solutions set of the system $(S): A \Box_{\min}^{\max} x = b$, which we denote by $\Lambda_b$ and the set $\mathcal{C}_b$ of Chebyshev approximations of the second member $b$, see (\ref{def:chebyshevsetapproxB}). By definition, an approximate solution $x^\ast \in \Lambda_b$ is a column vector such that the vector $c= A \Box_{\min}^{\max} x^\ast$ is a Chebyshev approximation of $b$ i.e., $c \in \mathcal{C}_b$. Moreover, for all $c \in \mathcal{C}_b$,   the solutions of the  system $A \Box_{\min}^{\max} x=c$ belong to $\Lambda_b$ i.e., they are approximate solutions.

This section is structured as follows. We begin by defining the set $\Lambda_b$ and we relate it to the set $\mathcal{C}_b$ (Subsection \ref{subsec:approxsolset}). We show that the set $\Lambda_b$ is non-empty (Proposition \ref{prop:nonemptylambdab}) and has a greater element (Proposition \ref{prop:greatestaprox}). In (Subsection \ref{subsec:characterizations}), we then give a characterization of $\Lambda_b$ (Proposition \ref{proposition:firstlambdab}) and we describe the structure of $\Lambda_b$ in (Theorem \ref{th:2}).  These two results allow us to study the minimal elements of $\mathcal{C}_b$ and $\Lambda_b$ (Subsection \ref{subsec:minicheb}).  Using  the results and the method of \cite{matusiewicz2013increasing} to construct minimal elements of systems of $\max-\min$ inequalities, we show that the set of minimal Chebyshev approximations $\mathcal{C}_{b,\min}$ is non-empty and finite. We also give a finite set of minimal approximate solutions noted $\Lambda_{b,\min}$ associated to $\mathcal{C}_{b,\min}$ by the following equality: $\mathcal{C}_{b,\min} = \{ A \Box_{\min}^{\max}x \mid x \in \Lambda_{b,\min} \}$. Finally, we describe the structure of the set $\mathcal{C}_b$ of Chebyshev approximations of $b$ (Theorem \ref{th:3}).

\subsection{Approximate solutions set \texorpdfstring{$\Lambda_b$}{Λb}}
\label{subsec:approxsolset}
We introduce a new notation and a new application:
\begin{notation}
    $\Lambda = {[0,1]}^{m\times 1}$.
\end{notation}
\begin{proposition}

\begin{equation}
    \theta: \Lambda  \rightarrow \mathcal{C}: x \mapsto A\Box_{\min}^{\max} x
\end{equation}
\noindent where $\mathcal{C}$ is defined in ( \ref{def:setofsecondmembersB}). The application $\theta$ is a \textit{surjective} and increasing map from $\Lambda$ onto $\mathcal{C}$.

\begin{proof}
By (Lemma \ref{lemma:increasing}), we know that the map $\theta$ is increasing.  As any $c\in {\cal C}$ gives rise to a consistent system $A \Box_{\min}^{\max} x = c$, the map $\theta$ is surjective.
\end{proof}
   
\end{proposition}

\noindent We remark that:
\begin{itemize}
    \item We have  $\mathcal{C} = \{ \theta(x) \mid x \in \Lambda\}$.
    \item For any $u \in [0,1]^{n \times 1}$ we have: 
\begin{equation}\label{eq:reformulationidempot} F(u) = \theta(e) \text{ where } e = A^t \Box_{\rightarrow_G}^{\min} F(u). 
\end{equation}
\noindent  This is an equivalent reformulation of the idempotence property of the application $F$, see (Proposition \ref{prop:threestatementF}).
\end{itemize}
    
\noindent In the following, we  introduce  the subset $\Lambda_b \subseteq \Lambda$ 
which is  the reciprocal image of the set 
$\mathcal{C}_{b}$ by the map $\theta$, i.e.:
\begin{definition}\label{def:approxsetLambda}
The approximate solutions set of the system $(S)$ is:
    \begin{equation}
    \Lambda_b = \theta^{-1}(\mathcal{C}_{b}) = \{x\in \Lambda \mid \theta(x) \in \mathcal{C}_{b}\}. 
\end{equation}
\end{definition}

As $\theta : \Lambda   \rightarrow \mathcal{C}$ is a \textit{surjective} map, we have:
\begin{equation}\label{eq:cbtolambdab}
    \mathcal{C}_b = \{ \theta(x) \mid x \in \Lambda_b \}.
\end{equation}
\noindent We define:  
\begin{definition}\label{def:approxsol}
A vector  $x \in \Lambda_b$ is called an \textit{approximate solution} of the system  $(S)$. \\By definition of the set $\Lambda_b$, we have for all $x \in \Lambda$:
\begin{equation}\label{eq:approxsol}
    x \in \Lambda_b \Longleftrightarrow \Vert A \Box_{\min}^{\max} x - b \Vert = \Delta(A,b). 
\end{equation}
\noindent where $\Delta(A,b)$ is the Chebyshev distance associated to the second member $b$ of the system $(S)$, see (Theorem \ref{th:Deltamin}). 
\end{definition}
(see other definitions of approximate solutions using another choice of norms in \cite{WEN2022,wu2022analytical}). \\
We have:
\begin{proposition}\label{prop:nonemptylambdab}
    The approximate solutions set $\Lambda_b$ is non-empty.
\end{proposition}
\begin{proof}
     As we know   by (Proposition \ref{proposition:greatestcheb}) that the set $\mathcal{C}_b$ is non-empty, we conclude by (\ref{eq:cbtolambdab}) that the set $\Lambda_b$ is also non-empty.  
\end{proof}

In fact, we have a particular (and important) element in $\Lambda_b$:
\begin{notation}
$\eta := A^t \Box_{\rightarrow_G}^{\min} F(\overline{b}(\Delta))$.
\end{notation}
 \begin{proposition}\label{prop:greatestaprox}
      The column vector  $\eta$ satisfies the equality $\theta(\eta)=F(\overline{b}(\Delta))$, therefore $\eta \in \Lambda_b$ and $\eta$  is the greatest approximate solution, i.e., the greatest element of $\Lambda_b$. 
 \end{proposition}
\begin{proof}
 \noindent The equality $\theta(\eta) = F(\overline{b}(\Delta))$  follows from  the idempotence property of the application $F$ (Proposition \ref{prop:threestatementF}). As  by (Proposition \ref{proposition:greatestcheb}), $F(\overline{b}(\Delta))\in {\cal C}_b$, we obtain that $\eta\in \Lambda_b$, see (Definition \ref{def:approxsetLambda}).

 \noindent Let us show $x \in \Lambda_b \Longrightarrow x \leq \eta$.  \\
Set $c = \theta(x)$ and $c' = F(\overline{b}(\Delta))$. As $c = \theta(x) \in \mathcal{C}_b$, we have $c \leq F(\overline{b}(\Delta)) =c'$ (Proposition \ref{proposition:greatestcheb}). We apply (Lemma \ref{lemma:appF}) to obtain $x \leq A^t \Box_{\rightarrow_G}^{\min} F(\overline{b}(\Delta)) = \eta$.\\
\end{proof}

\noindent
In what follows, we shall look  for a finite non-empty set denoted $\Lambda_{b,\min}$  of minimal approximate solutions, which satisfies: 
\begin{equation}\label{eq:rellambdabmincbmin}
    \Lambda_{b,\min} \subseteq \Lambda_b  \text{ and } \mathcal{C}_{b,\min} = \{ \theta(x) \mid x \in \Lambda_{b,\min} \}.
\end{equation}
The existence of such a set  $\Lambda_{b,\min}$,  which implies that the set  $\mathcal{C}_{b,\min}$ is also non-empty and finite, will be deduced from a    characterization  (Proposition \ref{proposition:firstlambdab})  of the set $\Lambda_b$ and a sharp result of \cite{matusiewicz2013increasing}  on  the solving of a system of inequalities.

\subsection{Characterizing  the approximate solutions set \texorpdfstring{$\Lambda_b$}{Λb}}
\label{subsec:characterizations}

\noindent We give the following first characterization of $\Lambda_b$:
\begin{proposition}\label{proposition:firstlambdab}
For any $x\in \Lambda$, we have:
    \begin{equation}\label{rel:lambdachevB}
 x \text{ is an approximate solution i.e., } x\in  \Lambda_{b}
\Longleftrightarrow 
\underline{b}(\Delta) \leq \theta(x) \text{ and } x \leq \eta.
 \end{equation} 
\end{proposition}
\begin{proof}\mbox{}\\
\noindent $\Longrightarrow$ 

\noindent $\bullet$ We know from  (Proposition \ref{prop:greatestaprox}) that $x \in \Lambda_b \Longrightarrow x \leq \eta$.

\noindent $\bullet$  Let us show $x \in \Lambda_b \Longrightarrow \underline{b}(\Delta) \leq \theta(x)$.\\
If $x \in \Lambda_b$, then $\Vert b-\theta(x) \Vert = \Delta$ which implies that for any $i \in \{1,2,\dots,n\}$, $b_i - {\theta(x)}_i \leq \Delta$  rewritten as $b_i - \Delta \leq {\theta(x)}_i$. As ${\theta(x)}_i \geq 0$, we deduce that:
\[ \forall i \in \{1,2,\dots,n\},  (b_i - \Delta)^+ \leq {\theta(x)}_i.\]

\noindent $\Longleftarrow$ \\We suppose $\underline{b}(\Delta) \leq \theta(x) \text{ and } x \leq \eta$ and we must prove $\Vert  b-\theta(x) \Vert = \Delta$. 
As $x \leq \eta$ and $\theta$ is increasing, we have $\theta(x) \leq \theta(\eta) = F(\overline{b}(\Delta))$. \\
\noindent As $F(\overline{b}(\Delta)) \in \mathcal{C}_b$ (Proposition \ref{proposition:greatestcheb}), then, for any $i \in \{1,2,\dots,n\}$, we have:
\[ -\Delta \leq b_i - {\theta(\eta)}_i \leq b_i - {\theta(x)}_i.  \]
\noindent On the other hand, $\underline{b}(\Delta) \leq \theta(x)$ implies that for any $i \in \{1,2,\dots,n\}$:
\[  b_i - {\theta(x)}_i \leq \Delta.  \]
\noindent In conclusion, we have for $i \in \{1,2,\dots,n\}$:
\[ -\Delta \leq b_i - {\theta(x)}_i \leq \Delta, \]
\noindent which is equivalent to $\Vert b - \theta(x) \Vert \leq \Delta$. But, $\theta(x) \in \mathcal{C}$, so $\Vert b - \theta(x) \Vert \geq \Delta$. Therefore, $\Vert b - \theta(x) \Vert = \Delta$ i.e., $x \in \Lambda_b$.
\end{proof}

\noindent To introduce a sharp characterization of $\Lambda_b$, which describes completely its structure, we will first give some notations and a lemma.

\begin{notation}\mbox{}
\begin{itemize}

    \item For $j = 1 , 2 , \dots ,m$, let:  
\[ H_j = \{i \in \{1,2,\dots,n\} \mid a_{ij} < b_i - \Delta\},
\]
\noindent

\item For any $T \subseteq \{1 , \dots , m\}$, we denote by $T^c$  the complement of $T$.
\item For $T \subseteq \{1 , \dots , m\}$, we put:
\begin{equation*}
    I_T= \bigcap_{j\in T} H_j \text{ and } \xi_T = \max_{i \in I_T}(b_i - \Delta)^+,
\end{equation*}   
\noindent with the convention $\max_\emptyset = 0$.
    \end{itemize}
\end{notation}
\noindent The map $T \mapsto \xi_T$ has the following properties:
\begin{lemma}\label{lemma:txit}\mbox{}
    \begin{enumerate}
        \item For $T = \emptyset$,  we have $\xi_\emptyset = \max_{i \in \{1,2,\dots,n\}}(b_i - \Delta)^+$.
        \item The map $T \mapsto \xi_T$ is decreasing i.e., $T \subseteq T' \Longrightarrow \xi_{T'} \leq \xi_{T}$.
        \item $\xi_{\{1,2,\dots,m\}}=0$.
    \end{enumerate}
\end{lemma}
\begin{proof}\mbox{}
\begin{enumerate}
    \item  This is true because $I_\emptyset = \bigcap\limits_{j \in \emptyset} H_j = \{1,2,\dots,n\}$. 
    \item If $T \subseteq T'$, then $I_{T'} \subseteq I_T$ and $\xi_{T'} = \max_{i \in I_{T'}}(b_i - \Delta)^+ \leq \xi_T = \max_{i \in I_T}(b_i - \Delta)^+$.
    \item From  (Proposition \ref{prop:greatestaprox}) and (Proposition \ref{proposition:firstlambdab}), we deduce:

    \[\underline{b}(\Delta) \leq \theta(\eta). \]
    Let us show $I_{\{1,2,\dots,m\}} = \emptyset$. 
    
    \noindent
In fact, for any $1 \leq i \leq n$, the inequality ${\underline{b}(\Delta)}_i \leq {\theta(\eta)}_i$ implies that there exists $1 \leq j \leq m$ such that: $${\underline{b}(\Delta)}_i = (b_i - \Delta)^+ \leq \min(a_{ij}, \eta_j) \leq a_{ij}.$$
    \noindent Then, $i \notin H_j$, so $I_{\{1,2,\dots,m\}}=\bigcap\limits_{j \in \{1,2,\dots,m\}} H_j=\emptyset$ and  by the convention   $\max_\emptyset = 0$,    we have   $\xi_{\{1,2,\dots,m\}}=0$.
\end{enumerate}
\end{proof}

\noindent The main characterization of the approximate solutions set $\Lambda_b$ is:
\begin{theorem}\label{th:2}
    For any  $x =\begin{bmatrix}{x}_j\end{bmatrix}_{1\leq j \leq m}  \in [0,1]^{m \times 1}$, we have:
    \begin{equation}\label{rel:lambdachevB2}
 x\in  \Lambda_{b}
\,\Longleftrightarrow \forall T \subseteq \{1,2,\dots,m\},\quad \xi_T \leq \max_{j \in T^c} x_j \text{ and } x \leq \eta.
 \end{equation} 
\end{theorem}
\noindent For the proof of (Theorem \ref{th:2}), we need first to establish for any  $x =\begin{bmatrix}{x}_j\end{bmatrix}_{1\leq j \leq m}  \in [0,1]^{m \times 1}$:
\begin{proposition}\label{proposition:bcij}
    $$\underline{b}(\Delta) \leq \theta(x) \Longleftrightarrow \forall i \in \{1,2,\dots,n\}, \exists j \in \{1,2,\dots,m\}, \text{ such that } i\in {H_j}^c \text{ and } (b_i - \Delta)^+ \leq x_j.$$
\end{proposition}

\begin{proof}
  \mbox{}\\
Let $i \in \{1,2,\dots,n\}$. We have:
\begin{align*}
    {(b_i - \Delta)^+} \leq {\theta(x)}_i  &\Longleftrightarrow  \exists j \in \{1,2,\dots,m\} \text{ such that } (b_i - \Delta)^+ \leq \min(a_{ij}, x_j)\\
    &\Longleftrightarrow \exists j \in \{1,2,\dots,m\} \text{ such that } (b_i - \Delta)^+ \leq a_{ij} \text{ and } (b_i - \Delta)^+  \leq  x_j\\
    &\Longleftrightarrow \exists j \in \{1,2,\dots,m\}, \text{ such that } i\in {H_j}^c \text{ and } (b_i - \Delta)^+ \leq x_j.
\end{align*}

\end{proof}
\noindent The proof of (Theorem \ref{th:2}) is given in the following. 
\begin{proof}
    \mbox{}\\
\noindent $\Longrightarrow$\\ 
We know by (Proposition \ref{proposition:firstlambdab}) that $x \leq \eta$.
Let $T \subseteq \{1,2,\dots,m\}$ and we must show $\xi_T \leq \max_{j \in T^c} x_j$.
\begin{itemize}
    \item If $I_T = \emptyset$, we have $\xi_T = 0\leq \max_{j \in T^c} x_j$.
    \item If $I_T \neq \emptyset$, then take  $i \in I_T$  such that $\xi_T = (b_i - \Delta)^+$. Using (Proposition \ref{proposition:bcij}), we have $j \in \{1,2,\dots,m\}$, such that $i\in {H_j}^c$ (which means that $(b_i - \Delta)^+ \leq a_{ij}$) and $(b_i - \Delta)^+ \leq x_j$.
    We conclude that $j \notin T$ i.e., $j \in T^c$ and:
    $$ \xi_T = (b_i - \Delta)^+ \leq x_j \leq  \max_{l \in T^c} x_l. $$
\end{itemize}

\noindent $\Longleftarrow$\\
To prove that $x\in\Lambda_b$, by (Proposition \ref{proposition:firstlambdab}), it is  sufficient to have $\underline{b}(\Delta) \leq \theta(x)$. Let $i \in \{1,2,\dots,n\}$,   we must show that $(b_i - \Delta)^+ \leq {\theta(x)}_i$.\\ Take $T =  \{ j \in \{1,2,\dots,m\} \mid i \in {H_j} \}$. Clearly, $i \in I_T$ and then: $$(b_i - \Delta)^+ \leq \xi_T \leq \max_{l \in T^c} x_l.$$
We distinguish two cases:
\begin{itemize}
    \item $T^c = \emptyset$, then $\max_{l \in T^c} x_l = 0$, and $(b_i - \Delta)^+ = 0 \leq {\theta(x)}_i$. 
    \item $T^c \neq \emptyset$, and let $l' \in T^c$ such that $x_{l'} = \max_{l \in T^c} x_{l}$. We have:
    $$  (b_i - \Delta)^+ \leq \xi_T  \leq  x_{l'}. $$
    \noindent But, $l' \in T^c$ means that $i \in {H_{l'}}^c$ i.e., $(b_i - \Delta)^+ \leq a_{il'}$. Finally, $$(b_i - \Delta)^+ \leq \min(a_{il'},x_{l'}) \leq {\theta(x)}_i.$$
\end{itemize}
 
\end{proof}

\subsection{Obtaining minimal Chebyshev approximations from minimal approximate solutions}
\label{subsec:minicheb}
\noindent From a practical point of view, one can obtain all the minimal Chebyshev approximations of the second member $b$ of the system $(S)$. For this purpose, we use (Proposition \ref{proposition:firstlambdab}) and the results of \cite{matusiewicz2013increasing}, where the authors showed that a system of $\max-\min$ relational inequalities  has a finite non-empty set of solutions, and they gave an algorithm to obtain the minimal solutions of such a system that are lower than a given solution. 

In the following, we use this result of \cite{matusiewicz2013increasing}:
\begin{notation}\label{notations:minimalsol} We denote by $\{ v^{(1)},v^{(2)},\dots,v^{(h)} \}$   the set of minimal solutions of the system of inequalities $\underline{b}(\Delta) \leq A  \Box_{\min}^{\max} x$ (obtained using the algorithm of \cite{matusiewicz2013increasing}) such that $\forall i \in \{1,2,\dots,h\}, v^{(i)} \leq \eta$.
\end{notation}
We have: 
\begin{proposition}\label{proposition:minimalsol}
   \mbox{}
\begin{enumerate}
\item $\{v^{(1)},v^{(2)},\dots,v^{(h)}\}\subseteq \Lambda_b\text{ and } \{ \theta(v^{(1)}), \theta(v^{(2)}), \dots, \theta(v^{(h)}) \} \subseteq \theta(\Lambda_b) = \mathcal{C}_{b}$,
 \item   $\forall x \in \Lambda_{b}, \exists i \in \{1 , 2 ,\dots , h\}, \text{ such that }  v^{(i)}  \leq x$, 
  \item  $\forall  c \in \mathcal{C}_{b,\min}, \exists i \in \{1, 2, \dots , h\}, \text{ such that } c = \theta(v^{(i)})$.
 \end{enumerate}  
\end{proposition} 
\begin{proof} For the proof of the first statement, we observe that $\{v^{(1)},v^{(2)},\dots,v^{(h)}\}\subseteq \Lambda_b$ is a consequence of   (Proposition~\ref{proposition:firstlambdab}) and $\{ \theta(v^{(1)}), \theta(v^{(2)}), \dots, \theta(v^{(h)}) \} \subseteq \theta(\Lambda_b) = \mathcal{C}_{b}$ is a consequence of  (\ref{eq:cbtolambdab}).

\noindent
To prove the second statement, let $x_0 \in \Lambda_{b}$.  From (Proposition \ref{proposition:firstlambdab}), we deduce:
$$ \underline{b}(\Delta) \leq A  \Box_{\min}^{\max} x_0 = \theta(x_0)\quad \text{and} \quad x_0 \leq\eta.$$ 
By the algorithm of \cite{matusiewicz2013increasing},   there is a minimal solution $v$ of the system of inequalities $\underline{b}(\Delta) \leq A  \Box_{\min}^{\max} x$ such that 
$v  \leq x_0$.
As we have $x_0 \leq\eta$,   we also have $v\leq \eta$, so there is an index $i\in\{1 , 2 , \dots , h\}$ such that $v= v^{(i)}  \leq x_0$.

\noindent
To prove the last statement, let $c\in \mathcal{C}_{b,\min}$. From (\ref{eq:cbtolambdab}), there is an element $x_0\in \Lambda_{b}$ such that  
$c = \theta(x_0)$ and 
from the second statement (of (Proposition \ref{proposition:minimalsol})), there is an index $i\in\{1 , 2 , \dots , h\}$ such that $  v^{(i)}  \leq x_0$.
 
\noindent
From the increasing  of 
$\theta$   and $\theta(v^{(i)}) \in \mathcal{C}_{b}$,  we deduce:
 $$\theta(v^{(i)}) \leq \theta(x_0) = c.$$ 
By minimality of $c$, we conclude that $c = \theta(v^{(i)})$.  
\end{proof}
\noindent The following corollary allows us to efficiently obtain the minimal Chebyshev approximations in practice.

\begin{corollary}\label{corollary:ctitlde}Using (Notation \ref{notations:minimalsol}),    we  put:
\begin{equation}\label{eq:ctilde}
    \widetilde{\mathcal{C}} = \{ \theta(v^{(1)}),\theta(v^{(2)}),\dots,\theta(v^{(h)}) \} 
\end{equation}
\noindent and 
\begin{equation}\label{eq:ctildemin}
    (\widetilde{\mathcal{C}})_{\min} = \{ c \in \widetilde{\mathcal{C}} \mid c \text{ is minimal in }\widetilde{\mathcal{C}}\}.
\end{equation}
 
\noindent Then, we have: \[\widetilde{\mathcal{C}} \subseteq \mathcal{C}_{b} \text{ and } \mathcal{C}_{b,\min} = (\widetilde{\mathcal{C}})_{\min}. \]
\begin{proof}
By the first statement of (Proposition \ref{proposition:minimalsol}), we have $ \widetilde{\mathcal{C}} \subseteq  \mathcal{C}_b $.

\noindent
By the third statement of (Proposition \ref{proposition:minimalsol}), we have $\, \mathcal{C}_{b,\min} \subseteq  \widetilde{\mathcal{C}}$. As
$ \widetilde{\mathcal{C}} \subseteq  \mathcal{C}_b $, we deduce   
$ \mathcal{C}_{b,\min} \subseteq (\widetilde{\mathcal{C}})_{\min}$. 

\noindent
Let $c\in (\widetilde{\mathcal{C}})_{\min}$. To prove that $c\in  \mathcal{C}_{b,\min} $, let 
$c'\in \mathcal{C}_{b}$ such that $c' \leq c$. We must prove that $c' =c$.

\noindent
By (\ref{eq:cbtolambdab}), there is an element $x_0\in \Lambda_b$ such that 
 $c' = \theta(x_0)$.

\noindent
Using the second statement of (Proposition \ref{proposition:minimalsol}), we obtain  an  index  $i \in \{1 , 2 ,\dots , h\}, \text{ such that }  v^{(i)}  \leq x_0$. Then we have $\theta(v^{(i)}) \in \widetilde{\mathcal{C} }$ and by the increasing of $\theta$, we get: 
$$\theta(v^{(i)}) \leq \theta(x_0) = c' \leq c.$$
By minimality of $c$ in  $\widetilde{\mathcal{C}}$, we obtain $\theta(v^{(i)})  = c$ , so $c' =c$.
\end{proof}
\end{corollary}
\noindent We have:
\begin{corollary}\label{cor:nonemptyfinite}
  The set  $\mathcal{C}_{b,\min}$ is non-empty and finite.
\end{corollary}
\begin{proof}
As    $\widetilde{\mathcal{C}}$ is a finite   non-empty ordered set, the set  
${(\widetilde{\mathcal{C}})}_{\min} = \mathcal{C}_{b,\min} $ is also finite  and  non-empty. 
\end{proof}
\noindent We are able to define a set of minimal approximation solutions $\Lambda_{b,\min}$, see (\ref{eq:rellambdabmincbmin}):
\begin{definition}
\begin{equation}
    \Lambda_{b,\min} = \{ x \in \{ v^{(1)}, v^{(2)},\dots, v^{(h)} \} \mid \theta(x) \in \mathcal{C}_{b,\min} \}.
\end{equation}
\end{definition}
It follows from the first and the last statements of (Proposition \ref{proposition:minimalsol}) that we have:
\begin{equation*}
    \Lambda_{b,\min} \subseteq \Lambda_b  \text{ and } \mathcal{C}_{b,\min} = \{ \theta(x) \mid x \in \Lambda_{b,\min} \}.
\end{equation*}
\noindent {\it Therefore, the set $\Lambda_{b,\min}$ is  non-empty and finite}.\\

\noindent
The structure of the set ${\cal C}_b$ is  described by the following result:
\begin{theorem}\label{th:3} For all $c\in [0 , 1]^{n \times 1}$, we have: 
\begin{equation}c \text{ is a Chebyshev approximation of } b \text{ i.e., } c\in {\cal C}_b \,\Longleftrightarrow \, F(c)  = c  \,\,\text{and}\,\,\exists \, c'\in {\cal C}_{b,\min} \,\,
\text{s.t.}\,\, c' \leq c \leq F(\overline{b}(\Delta)).\end{equation}
\end{theorem}
\begin{proof}\mbox{}\\
$\Longrightarrow$

\noindent
Let $c\in{\cal C}_b$. As ${\cal C}_b \subseteq {\cal C}$, we   know from (Proposition \ref{proposition:appFreformulated}) that $F(c) = c$. From  (\ref{eq:cbtolambdab}), we have an approximate solution $x_0 \in \Lambda_b$ such that $c = \theta(x_0)$. Then, by the  second statement of (Proposition \ref{proposition:minimalsol}), there is an index $i \in \{1 , 2 , \dots , h\}$ such that $v^{(i)} \leq  x_0$. Set $c_1 = \theta(v^{(i)})$. Then, by (Corollary \ref{corollary:ctitlde}), we have $c_1 \in \widetilde{\cal C}$ and there exist an element $c'\in \widetilde{\cal C}_{\min} = {\cal C}_{b , \min}$ such that $c' \leq c_1$. As $\theta$ is increasing and using (Proposition \ref{proposition:greatestcheb}), we have:
$$c' \leq c_1 = \theta(v^{(i)}) \leq \theta(x_0) = c  \leq  F(\overline{b}(\Delta)).$$
$\Longleftarrow$

\noindent
As $F(c) = c$, by (Proposition \ref{proposition:appFreformulated}), we have $c\in{\cal C}$. It remains us to prove that $\Vert b - c \Vert = \Delta$.

\noindent
Let $c' \in{\cal C}_{b,\min}$ such that $c' \leq c$. As we have: 
$$\Vert b - c' \Vert = \Vert b - F(\overline{b}(\Delta)) \Vert = \Delta \quad \text{and}\quad c' \leq c \leq F(\overline{b}(\Delta)),$$
we deduce for all $i\in \{1  , 2 , \dots , n\}:$ 
$$- \Delta \leq b_i  -  F(\overline{b}(\Delta))_i  \leq b_i - c_i \leq b_i - c'_i \leq \Delta,$$
so $\Vert b - c \Vert \leq \Delta$. As $c\in {\cal C}$, we have also $\Vert b - c \Vert \geq  \Delta$. Finally, we conclude $\Vert b - c \Vert = \Delta$, so $c$ is a Chebyshev approximation of $b$, i.e., $c \in \mathcal{C}_b$.
\end{proof}

We illustrate our method for obtaining the minimal Chebyshev approximations of $b$.

\begin{example}(continued) We continue with the results in (Example \ref{ex:example5}). \\We remind that the Chebyshev distance associated to the second member $b$ of the system $A \Box_{\min}^{\max} x = b$ is $\Delta = 0.16$.\\
We compute:
\begin{equation*}\label{eq:ex7:eta}
\underline{b}(\Delta) = \begin{bmatrix}
    0.38 \\
    0.00 \\
    0.71
    \end{bmatrix}, 
    \overline{b}(\Delta) = \begin{bmatrix}
0.70\\
0.29\\
1.00 \\
\end{bmatrix}\text{ and } \eta =  A^t \Box_{\rightarrow_G}^{\min} F(\overline{b}(\Delta)) =\begin{bmatrix}
      0.29\\
      1\\
      1
   \end{bmatrix}.
\end{equation*}    
The vector $\theta(\eta) = \begin{bmatrix}
0.38 \\
0.29\\
0.85 \\
\end{bmatrix}$ is the greatest Chebyshev approximation of the second member $b$ of the system.

\noindent The system of inequalities   $\underline{b}(\Delta) \leq A  \Box_{\min}^{\max} x$ is:
$
    \begin{bmatrix}
        0.38\\
        0.00\\
        0.71
    \end{bmatrix} \leq 
        \begin{bmatrix}
0.03&0.38&0.26\\
0.98&0.10&0.03\\
0.77&0.15&0.85\\
\end{bmatrix}\Box_{\min}^{\max} \begin{bmatrix}
    x_1\\
    x_2\\
    x_3\\
\end{bmatrix}$.\\
\noindent Using the approach of \cite{matusiewicz2013increasing}, we obtain two minimal solutions: $
   v=\begin{bmatrix}
   0.00 \\
   0.38 \\
   0.71 
   \end{bmatrix} \text{ and } 
   v'=\begin{bmatrix}
       0.71\\ 0.38\\ 0.00
   \end{bmatrix}$ of the system of inequalities. Among these minimal solutions, only $v$ is lower than $\eta$. \\The set $\widetilde{\mathcal{C}}$, see (\ref{eq:ctilde}), contains one element, which is $A \Box_{\min}^{\max} v   = \begin{bmatrix}0.38\\ 0.10\\ 0.71\end{bmatrix}$ and we have $\widetilde{\mathcal{C}} = (\widetilde{\mathcal{C}})_{\min}$. Therefore, from (Corollary \ref{corollary:ctitlde}),  
    the unique minimal Chebyshev approximation of $b$ is $\check{b} = \begin{bmatrix}0.38\\ 0.10\\ 0.71\end{bmatrix}$.\\

    \noindent Some approximate solutions of the system $(S)$ are  the solutions of the system $\theta(\eta) = A \Box_{\min}^{\max} x$  and the  solutions of the system $\check{b} = A \Box_{\min}^{\max} x$.  
\end{example}

\section{Learning approximate weight matrices according to training data}
\label{sec:learning}
\noindent Numerous approaches have been proposed for learning a weight matrix relating input data to output data by $\max-\min$ composition  \cite{blanco1994solving,blanco1995identification,blanco1995improved,ciaramella2006fuzzy,de1993neuron,hirota1982fuzzy,hirota1996solving,hirota1999specificity,ikoma1993estimation,li2017convergent,pedrycz1983numerical,pedrycz1991neurocomputations,pedrycz1995genetic,saito1991learning,stamou2000neural,teow1997effective,zhang1996min}. One of the pioneering works is that of Pedrycz \cite{pedrycz1991neurocomputations}. He highlighted that we can represent a system of $\max-\min$ fuzzy relational equations $W \Box_{\min}^{\max} x = y$ by a neural network, where $W = \begin{bmatrix}
    w_{ij}
\end{bmatrix}_{1 \leq i \leq n, 1 \leq j \leq m} \in [0,1]^{n \times m}$ is called a weight matrix and  $x = \begin{bmatrix}
    x_j
\end{bmatrix}_{1\leq j \leq m} \in [0,1]^{m \times 1}$  and $y = \begin{bmatrix}
    y_i
\end{bmatrix}_{1 \leq i \leq n} \in [0,1]^{n \times 1}$ are column vectors. 
The neural network (Figure \ref{fig:nnNew}) has $m$ input nodes corresponding to the components $x_1, x_2,\cdots, x_m$ of $x$, $n$ output nodes corresponding to the components $y_1, y_2, \cdots, y_n$ of $y$ and $n\cdot m$ edges such that each of the edges is weighted by the component  $w_{ij}$ of  $W$, and connects the input node  $x_j$ to the output node  $y_{i}$. For $1 \leq i \leq n$, the value of the output node  $y_i$ is given by  $y_i = \max_{1 \leq j \leq m} \min(w_{ij},x_j)$. \\

\begin{figure}[H]
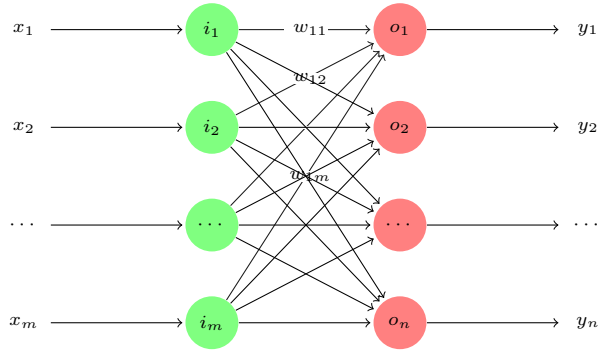

\centering

\scriptsize 
    	\begin{neuralnetwork} [nodespacing=13mm, layerspacing=25mm,
			maintitleheight=0em, layertitleheight=0em,
			height=0, toprow=false, nodesize=20pt, style={},
			title={}, titlestyle={}]
			
	 \newcommand{\nodetextinput}[2]{
	    \ifnum1=#2 $i_{1}$   \fi 
		\ifnum2=#2 $i_{2}$ \fi 
		\ifnum3=#2 $\cdots$ \fi 
		\ifnum4=#2 $i_{m}$ \fi 
	    }
	    
	     \newcommand{\nodetextinputvar}[2]{
	    \ifnum1=#2 $x_{1}$   \fi 
		\ifnum2=#2 $x_{2}$ \fi 
		\ifnum3=#2 $\cdots$ \fi 
		\ifnum4=#2 $x_{m}$ \fi 
	    }
	    
	      \newcommand{\nodetextoutputvar}[2]{
	    \ifnum1=#2 $y_{1}$   \fi 
		\ifnum2=#2 $y_{2}$ \fi 
		\ifnum3=#2 $\cdots$ \fi 
		\ifnum4=#2 $y_{n}$ \fi 
	    }

		\newcommand{\nodetextoutput}[2]{
	    \ifnum1=#2 $o_{1}$   \fi 
		\ifnum2=#2 $o_{2}$ \fi 
		\ifnum3=#2 $\cdots$ \fi 
		\ifnum4=#2 $o_{n}$ \fi 
	    }

		\newcommand{\linklabelsEdgeO}[4]{
		\ifnum1=#2 $w_{11}$   \fi 
		\ifnum2=#2 $w_{12}$ \fi 
		\ifnum4=#2 $w_{1m}$ \fi
		}

		\tikzstyle{inputoutput}=[neuron, fill=none];
		
\layer[count=4, bias=false,nodeclass=inputoutput,text=\nodetextinputvar]
	    
\inputlayer[count=4,bias=false,text=\nodetextinput]	
\link[from layer=0, to layer=1, from node=1, to node=1]
\link[from layer=0, to layer=1, from node=2, to node=2]
\link[from layer=0, to layer=1, from node=3, to node=3]
\link[from layer=0, to layer=1, from node=4, to node=4]

\outputlayer[count=4,bias=false, text=\nodetextoutput]
	\link[from layer=1, to layer=2, from node=1, to node=1, label=\linklabelsEdgeO,style=black!90]
  \link[from layer=1, to layer=2, from node=2, to node=1, label=\linklabelsEdgeO,style=black!90]
  \link[from layer=1, to layer=2, from node=3, to node=1, label=\linklabelsEdgeO,style=black!90]
   \link[from layer=1, to layer=2, from node=4, to node=1, label=\linklabelsEdgeO,style=black!90]
  
  \link[from layer=1, to layer=2, from node=1, to node=2]
  \link[from layer=1, to layer=2, from node=2, to node=2]
  \link[from layer=1, to layer=2, from node=3, to node=2]
  \link[from layer=1, to layer=2, from node=4, to node=2]
  
  \link[from layer=1, to layer=2, from node=1, to node=3]
  \link[from layer=1, to layer=2, from node=2, to node=3]
  \link[from layer=1, to layer=2, from node=3, to node=3]
  \link[from layer=1, to layer=2, from node=4, to node=3]
  
   \link[from layer=1, to layer=2, from node=1, to node=4]
  \link[from layer=1, to layer=2, from node=2, to node=4]
  \link[from layer=1, to layer=2, from node=3, to node=4]
  \link[from layer=1, to layer=2, from node=4, to node=4]
 
 \layer[count=4, bias=false,nodeclass=inputoutput,text=\nodetextoutputvar]
 
 \link[from layer=2, to layer=3, from node=1, to node=1]
\link[from layer=2, to layer=3, from node=2, to node=2]
\link[from layer=2, to layer=3, from node=3, to node=3]
\link[from layer=2, to layer=3, from node=4, to node=4]
 
\end{neuralnetwork}
		\caption{A system of $\max-\min$ fuzzy relational equations represented by a $\max-\min$ neural network. Green nodes are input nodes and red nodes are output nodes.}
		\label{fig:nnNew} 
\end{figure}
\noindent To learn the weight matrix $W$ according to training data, most of the approaches try to adapt the classical  gradient descent method to such a $\max-\min$ fuzzy neural network in order to minimize  the   learning error $E(W)$ expressed in the   $L_2$ norm.  However,  since  the functions $\max$ and $\min$ are not fully differentiable, it is rather   difficult to adapt the classical gradient descent to this framework. This issue was recently again encountered in  \cite{van2022analyzing}.
\noindent In these approaches, it seems that the choice of the $L_2$ norm is motivated by its adequacy to the differentiable calculus, while being equivalent to the $L_\infty$ norm (two norms on the vector space $\mathbb{R}^n$ are equivalent).

\noindent In this section, based on our results, we introduce a paradigm to approximately learn a weight matrix relating input and output data from the following training data:
\begin{equation}\label{def:trainingdataIntro}
(x^{(i)})_{1\leq i \leq N},  x^{(i)} \in [0 , 1]^{m \times 1} \quad ; \quad 
(y^{(i)})_{1\leq i \leq N},  y^{(i)} \in [0 , 1]^{n \times 1}. \end{equation}
\noindent For $i=1,2,\dots,N$, each pair $(x^{(i)},y^{(i)})$ is a training datum, where $x^{(i)}$ is the input data vector and $y^{(i)}$ is the targeted output data vector. Our choice of norm to express the learning error is the $L_\infty$ norm:
\begin{equation}\label{eq:ew}
    E(W) = \max_{1 \leq i \leq N} \Vert  y^{(i)} - W \Box_{\min}^{\max} x^{(i)} \Vert   
\end{equation}
\noindent \textit{where the norm of a vector $z$ of $n$ components is $\Vert z \Vert =   \max_{1 \leq k \leq n} \mid z_k \mid$}. \\
\noindent The first main result of this section is that we can compute by an analytical formula a positive constant $\mu$, which depends only on the training data, such that the following equality holds: 
\begin{equation}\label{eq:muequalminW}
    \mu = \min_{W \in {[0,1]^{n \times m}}} E(W).
\end{equation}
\noindent In other words, our positive constant $\mu$ minimizes the learning error. 
 Whatever if $\mu=0$ or $\mu > 0$, we give a method to get a weight matrix $W^\ast$ such that $E(W^\ast) = \mu$. If $\mu = 0$, this method is based on the solving of  $n$ consistent systems of $\max-\min$ fuzzy relational equations constructed from the training data that we will introduce. Otherwise, if $\mu > 0$, we get an approximate weight matrix $W^\ast$ by gathering  approximate solutions (Definition \ref{def:approxsol}) of these same systems using (Section \ref{sec:relating}).

\noindent This section is structured as follows. Considering a training data where the outputs are scalar (one value), we begin by relating the problem of learning a weight matrix connecting input data to output data  to the solving of a system canonically associated to this training data (Subsection \ref{subsec:relatingchelou}). Then, we tackle the general problem  (Subsection \ref{subsec:generalcase}). After giving some notations (Subsection \ref{subsec:notationslearning}) and defining the positive constant $\mu$ (Definition \ref{def:mu}) of  (Subsection \ref{subsec:defest}), we prove (\ref{eq:muequalminW}) and  give a method ((Method \ref{method:defou}) of (Subsection \ref{subsec:method})) for constructing approximate weight matrices  i.e., matrices $W$ such that $E(W) = \mu$. Finally, in (Subsection \ref{subsec:example}), we illustrate our results with two examples.

\subsection{Relating the problem of learning a weight matrix connecting input data to output data to the solving of a system canonically associated to these data}
\label{subsec:relatingchelou}

\noindent Assume a training data composed of $N$ piece of data as follows:
\begin{equation}\label{def:trainingdatasingleoutput}
(x^{(i)})_{1\leq i \leq N},  x^{(i)} \in [0 , 1]^{m \times 1} \quad ; \quad 
(y^{(i)})_{1\leq i \leq N},  y^{(i)} \in [0 , 1]. \end{equation}
\noindent For $i=1,2,\dots,N$, each pair $(x^{(i)},y^{(i)})$ is a training datum, where $x^{(i)}$ is an input data vector and $y^{(i)}$ is the targeted output data value in $[0,1]$.

\noindent We want to learn a weight matrix $V \in [0,1]^{1 \times m}$ such that:
\begin{equation}\label{eq:problemxiyi}
    \forall i \in \{1,2,\dots,N\},\,\, V \Box_{\min}^{\max} x^{(i)} = y^{(i)}.
\end{equation}

\noindent To tackle this problem, the idea is to introduce the following  system which is canonically associated to the training data: 
\begin{equation}\label{eq:slub}
    (S): L \Box_{\min}^{\max} u = b,
\end{equation}
\noindent where:
\begin{equation}\label{eq:see38L}
    L = \begin{bmatrix}
    x_j^{(i)}
\end{bmatrix}_{1 \leq i \leq N, 1 \leq j \leq m}=\begin{bmatrix}
    x_{1}^{(1)} & x_{2}^{(1)} & \cdots & x_{m}^{(1)} \\
    x_{1}^{(2)} & x_{2}^{(2)} & \cdots & x_{m}^{(2)} \\
    \vdots & \vdots & \vdots & \vdots \\
     x_{1}^{(N)} & x_{2}^{(N)} & \cdots & x_{m}^{(N)} 
    \end{bmatrix}\quad \text{and} \quad b = [y^{(i)}]_{1 \leq i \leq N} = \begin{bmatrix}
        y^{(1)}\\
        y^{(2)}\\
        \vdots\\
        y^{(N)}
    \end{bmatrix}.
\end{equation} 
    So the rows of $L$ are the transpose of the input data column vectors  $x^{(1)}, x^{(2)}, \dots,x^{(N)}$ and the components of $b$ are the targeted output values $y^{(1)},y^{(2)},\dots,y^{(N)}$.

\noindent To relate the problem formulated in (\ref{eq:problemxiyi}) to the system $(S)$, we will use the following lemma: 
\begin{lemma}\label{lemma:v:rowmatrix}
Let $v = \begin{bmatrix}
v_1\\ v_2 \\ \vdots \\ v_m \end{bmatrix} \in[0 , 1]^{m \times 1}$ be a column-vector and $V =     \begin{bmatrix}
v_1& v_2& \cdots & v_m  \end{bmatrix}\in[0 , 1]^{1 \times m}$ is the row matrix which is the transpose of $v$. 
We put  $v' = \begin{bmatrix}v'_i\end{bmatrix}_{1\leq i \leq N} =  L    \Box_{\min}^{\max} v$. Then, we have: 
\begin{enumerate} 
\item $\forall i \in \{1 , 2 , \dots ,N\}\,,\,v'_i = V \Box_{\min}^{\max} x^{(i)}\in [0 , 1]$,
\item $\Vert  b - v' \Vert = \max_{1 \leq i \leq N} \mid   y^{(i)} - V \Box_{\min}^{\max} x^{(i)} \mid$.
\end{enumerate}
The second statement implies that $V$  is a weight matrix of the  training data $((x^{(i)})_{1\leq i \leq N}, (y^{(i)})_{1 \leq i \leq N})$, see (\ref{eq:problemxiyi}),  if and only if $v$ is a solution of the system $(S)$.
 \end{lemma}
\begin{proof}  We have:
\begin{align} \forall i \in \{1 , 2 , \dots ,N\}\,,\,v'_i & = \max_{1\leq j \leq m}\, \min(l_{ij} , v_j) \nonumber\\
& =    \max_{1\leq j \leq m}\, \min(x^{(i)}_j , v_j) \nonumber\\
& = V \Box_{\min}^{\max} x^{(i)}.\nonumber
\end{align}
From these computations, we deduce the second statement: 
\begin{equation*}
    \Vert b - v'   \Vert  =   \max_{ 1 \leq i  \leq N}\, \,\mid    y^{(i)}-v'_i\mid= \max_{ 1 \leq i  \leq N}\,\mid y^{(i)} - V \Box_{\min}^{\max} x^{(i)}    \mid. 
\end{equation*}
\end{proof} 
\noindent The  problem formulated in (\ref{eq:problemxiyi}) is related to the system $(S)$ by: 
\begin{proposition}
Let $v = \begin{bmatrix}
v_1\\ v_2 \\ \vdots \\ v_m \end{bmatrix} \in[0 , 1]^{m \times 1}$ be a column-vector and $V =     \begin{bmatrix}
v_1& v_2& \cdots & v_m  \end{bmatrix}\in[0 , 1]^{1 \times m}$ is the row matrix which is the transpose of $v$. We have:
 \[ v \,\,\text{is a solution of the system $(S)$}\, \Longleftrightarrow\, 
\forall i \in \{1 , 2 , \dots ,N\}\, V \Box_{\min}^{\max} x^{(i)} =  y^{(i)}.\]
\end{proposition}
\begin{proof} 
The proof of this proposition follows directly from the second statement of (Lemma \ref{lemma:v:rowmatrix}).  
\end{proof} 
\noindent We have:
\begin{remark}The transpose map $[0,1]^{m \times 1} \rightarrow [0,1]^{1 \times m}: v \mapsto V = v^t$ defines a bijective correspondence between solutions of the system $(S)$ and weight matrices associated to the training data. 
\end{remark}

\noindent In the case where the system $(S): L \Box_{\min}^{\max} u = b$ is inconsistent,  we will show that the transpose map still defines a bijective correspondence between  approximate solutions of the system $(S)$ (Definition \ref{def:approxsol}) and approximate weight matrices $V$ i.e., matrices satisfying the following equality:
\begin{equation}\label{eq:defapproxV}
     \max_{1 \leq i \leq N}\, \mid y^{(i)} - V\Box_{\min}^{\max} x^{(i)}\mid  = \Delta(L,b),
\end{equation}
\noindent where $\Delta(L,b)$ is the Chebyshev distance associated to the second member $b$ of the system $(S)$, see (Definition \ref{def:chebyshevdist}). 

\noindent The definition (\ref{eq:defapproxV}) of an approximate weight matrix $V$ is justified by:
\begin{enumerate}
\item  For any approximate solution $v \in [0,1]^{m \times 1}$ of the system $(S)$, see (\ref{eq:slub}), we have  $\Vert b - L \Box_{\min}^{\max} v   \Vert = \Delta(L,b)$ (Definition \ref{def:approxsol}). 
    \item  It follows from the second statement of (Lemma \ref{lemma:v:rowmatrix})  and (Definition \ref{def:chebyshevdist}) that for any $V \in [0,1]^{1 \times m}$,  we have: 
    \begin{equation}\label{eq:onarrivealamaminimi}
        \max_{1 \leq i \leq N}\, \mid y^{(i)} - V\Box_{\min}^{\max} x^{(i)}\mid = \Vert b - L\Box_{\min}^{\max} v \Vert \geq \Delta(L,b).
    \end{equation}

\end{enumerate}
\noindent This leads to the  definition of the positive constant $\mu$:
\begin{definition}
    The positive constant $\mu$ minimizing the learning error $E(V)=\max_{1 \leq i \leq N} \mid  y^{(i)} - V \Box_{\min}^{\max} x^{(i)} \mid$, see (\ref{eq:ew}), according to the training data, is the Chebyshev distance associated to the second member $b$ of the system $(S)$:
    \begin{equation}
        \mu = \Delta(L,b). 
    \end{equation}
\end{definition}
\noindent This definition is justified by (\ref{eq:onarrivealamaminimi}), which we rewrite as \begin{equation}
    \forall V \in [0,1]^{1 \times m}, \, E(V) \geq \mu.
\end{equation}\\ 
To get the equality $
\mu = \min_{V \in {[0,1]^{1 \times m}}} E(V)$, see (\ref{eq:muequalminW}), we establish the following result:
\begin{proposition}
    Let $v = \begin{bmatrix}
v_1\\ v_2 \\ \vdots \\ v_m \end{bmatrix} \in[0 , 1]^{m \times 1}$ be a column-vector and $V =     \begin{bmatrix}
v_1& v_2& \cdots & v_m  \end{bmatrix}\in[0 , 1]^{1 \times m}$ is the row matrix which is the transpose of $v$. We have:
 \[ v \,\,\text{is an approximate solution of the system $(S)$}\, \Longleftrightarrow\, 
 \max_{1 \leq i \leq N}\, \mid y^{(i)} - V\Box_{\min}^{\max} x^{(i)}\mid  =  \Delta(L,b)=\mu.\]
\end{proposition}
\begin{proof}
   This equivalence is deduced from the second statement of (Lemma \ref{lemma:v:rowmatrix}) and the equivalence (\ref{eq:approxsol}).
\end{proof}
\noindent We deduce:
\begin{corollary}The equality 
    $
\mu = \min_{V \in {[0,1]^{1 \times m}}} E(V)$ holds.
\end{corollary}
\begin{proof}
    This result is a consequence of the fact that the approximate solution set $\Lambda_b$  is non-empty, see (Proposition~\ref{prop:nonemptylambdab}). 
\end{proof}

We illustrate this construction. 
\begin{example}
\noindent Let us consider the following training data:
\begin{table}[H]
    \centering
    \begin{tabular}{c|c}
       $x^{(1)} =(0.7, 0.4, 0.4)^t$  & $y^{(1)} =0.7$  \\
       $x^{(2)} =(1.0, 0.2, 0.5)^t$  & $y^{(2)} =1.0$  \\ 
       $x^{(3)} =(0.2, 0.3, 0.8)^t$  & $y^{(3)} =0.3$  
    \end{tabular}
    \caption{Training data. We have $N=3, m = 3$.}
    \label{tab:learning:ex22}
\end{table}
We construct the system $(S): L \Box_{\min}^{\max} u= b$ where $L= \begin{bmatrix}0.7 & 0.4 & 0.4  \\ 1.0 & 0.2 & 0.5  \\ 0.2 & 0.3 & 0.8  \end{bmatrix}$ and $b = \begin{bmatrix}0.7 \\ 1.0 \\ 0.3 \end{bmatrix}$. The system is consistent because the Chebyshev distance associated to $b$ is equal to zero:  $\Delta(L,b)=0$, so $\mu= \Delta(L,b)=0$. The greatest solution of $(S)$ is $\begin{bmatrix}
    1.0\\ 1.0\\ 0.3
\end{bmatrix}$ and there are two  minimal solutions $\begin{bmatrix}1.0\\ 0.3\\ 0.0\end{bmatrix}$ and $\begin{bmatrix}1.0\\ 0.0\\ 0.3\end{bmatrix}$ computed using the algorithm of \cite{matusiewicz2013increasing}. Let us use the solution $v = \begin{bmatrix}
    1.0\\ 0.7\\ 0.3
\end{bmatrix}$ of the system $(S)$ and we put $V = v^t=\begin{bmatrix}
    1.0&  0.7& 0.3
\end{bmatrix}$.  The weight matrix $V$ relates input and output data of the training data:
\begin{gather*}
    V\Box_{\min}^{\max} x^{(1)} = y^{(1)}, \\
    V \Box_{\min}^{\max} x^{(2)} = y^{(2)},\\ 
    V \Box_{\min}^{\max} x^{(3)} = y^{(3)}.
\end{gather*}
\end{example}

\subsection{Learning approximate weight matrices  in the general case}
\label{subsec:generalcase}
\noindent We shall extend the above results in the case where the outputs of the training data are column vectors of $n$ components  in $[0,1]$.  Let us consider $N$ training datum as follows:
\begin{equation}\label{def:trainingdata}
(x^{(i)})_{1\leq i \leq N},  x^{(i)} \in [0 , 1]^{m \times 1} \quad ; \quad 
(y^{(i)})_{1\leq i \leq N},  y^{(i)} \in [0 , 1]^{n \times 1}. \end{equation}
\noindent For $i=1,2,\dots,N$, each pair $(x^{(i)},y^{(i)})$ is a training datum, where $x^{(i)}$ is the input data vector and $y^{(i)}$ is the targeted output data vector.

\noindent We study the following problems: 
\begin{enumerate}
\item Is there a weight matrix $W$ of size $(n,m)$ such that:
 \[\forall i \in\{1 , 2 , \cdots , N\},\,\, W \Box_{\min}^{\max} x^{(i)} = y^{(i)}.\]
\item If this not the case, how to define and get a suitable approximate weight matrix $W$ ?
\end{enumerate}

\noindent We will prove the following results:
\begin{enumerate}

\item There is a positive constant denoted $\mu$ which can be computed by an analytical formula according to the training data and which satisfies:
\begin{equation}\label{Dict43}
\forall W \in [0 , 1]^{n \times m}, \max_{1 \leq i \leq N}\, \Vert y^{(i)} -  W \Box_{\min}^{\max} x^{(i)}\Vert \geq \mu.
\end{equation}

\noindent This positive constant minimizes the learning error $E(W)$, see (\ref{eq:ew}), and is expressed in terms of Chebyshev distances associated to the second member of systems of $\max-\min$ fuzzy relational equations  that we will introduce.
\item We will show the following equivalence: having a weight matrix that perfectly relates the input data to the output data is equivalent to having $\mu=0$ i.e., 
\begin{equation}\label{Dict45}
\exists W \in [0 , 1]^{n \times m} , \text{ s.t. }\forall i \in\{1 , 2 , \cdots , N\}, W \Box_{\min}^{\max} x^{(i)} = y^{(i)}
\,\Longleftrightarrow \,
\mu = 0.
\end{equation}
\item We will show that the set of approximate weight matrices:
\begin{equation}\label{eq:learning:setA}
    {\cal A} = \bigg\{ W \in [0,1]^{n \times m} \, \bigl\vert \,\, \max_{1 \leq i \leq N}\, \Vert y^{(i)} - W \Box_{\min}^{\max} x^{(i)}\Vert = \mu \bigg\} 
\end{equation}
\noindent is non-empty. This implies that $\mu = \min_{W \in {[0,1]^{n \times m}}} E(W)$, see (\ref{eq:muequalminW}).

 \end{enumerate}

\noindent In the following, we begin by giving some notations, then we define the positive constant $\mu$ and introduce our method for constructing an approximate  weight matrix $W$ according to training data.

\subsection{Notations}
\label{subsec:notationslearning}
\noindent
We reuse the matrix $L= [l_{ij}]_{1 \leq i \leq N, 1 \leq j \leq m} = [x^{(i)}_j]_{1 \leq i \leq N, 1 \leq j \leq m}$ of size $(N,m)$, see (\ref{eq:see38L}), which is defined by the transpose of the input data column  vectors $x^{(1)}, x^{(2)}, \dots,x^{(N)}$.

\noindent To extend to the case where the output data are vectors of $n$ components, we associate to the training data $n$ systems of $\max-\min$ fuzzy relational equations denoted by $(S_1),(S_2),\dots,(S_n)$, which all use the same matrix $L$ and whose second members are $b^{(1)}, b^{(2)}, \dots, b^{(n)}$. For $1 \leq k \leq n$, the system $(S_k)$ is of the form:
\begin{equation}
    (S_k) : L \Box_{\min}^{\max} u= b^{(k)},
\end{equation}
\noindent
where the unknown part is a column vector $u \in [0,1]^{m \times 1}$ and for  $k=1,2,\dots,n$, the components of the column vector $b^{(k)}=[b_i^{(k)}]_{1 \leq i \leq N}$ are  defined by:
\begin{equation}\label{Dict462}
  b_i^{(k)} = y^{(i)}_k \text{ ; } 1 \leq i \leq N.
 \end{equation}
\noindent We remark that for $k=1,2,\dots,n$ and $i=1,2,\dots,N$, each component $b_i^{(k)}$  of the second member $b^{(k)}$ of the system $(S_k)$ is equal to the component   $y_k^{(i)}$ of the targeted output data vector $y^{(i)}$:
\begin{equation}\label{eq:yi}
    b^{(k)}= \begin{bmatrix}
    y_{k}^{(1)} \\ y_{k}^{(2)} \\ \vdots \\ y_{k}^{(N)}
    \end{bmatrix}.
\end{equation}

\noindent
To any matrix $W = [w_{kj}]_{1 \leq k \leq n , 1 \leq j \leq m}$, we associate the $n$ - tuple of  column-vectors 
$(u^{(1)} ,  u^{(2)} , \dots , u^{(n)})$ where for all $1\leq k\leq n$, the column vector $u^{(k)}= \begin{bmatrix}
    u^{(k)}_j
\end{bmatrix}_{1 \leq j \leq m}$ is the transpose of the $k$-th row of the matrix $W$: 
\begin{equation}\label{Dict47}
u^{(k)}_j  = w_{kj} \text{ ; }  1 \leq k \leq n, 1 \leq j \leq m.
\end{equation}
\noindent  This defines the following bijective map between the sets $[0 , 1]^{n \times m} \,$ and $ \, ([0 , 1]^{m \times 1})^n$: 
\begin{equation}\label{Dict477} [0 , 1]^{n \times m} \rightarrow ([0 , 1]^{m \times 1})^n : W \mapsto (u^{(1)} ,  u^{(2)} , \dots , u^{(n)})\end{equation} 
Every $n$-tuple $(u^{(1)} ,  u^{(2)} , \dots , u^{(n)})$ of  column-vectors in $[0 , 1]^{m \times 1}$ is  the image by the above map  of a unique matrix $W \in [0  , 1 ]^{n \times m}$. Graphically, if $(u^{(1)} ,  u^{(2)} , \dots , u^{(n)})$ is the image of $W$ by the above map, we have:
\begin{equation*}
    W = \begin{bmatrix}
    w_{11} & w_{12} & \cdots & w_{1m} \\[8pt]
    w_{21} & w_{22} & \cdots & w_{2m} \\[8pt]
    \vdots & \vdots & \vdots & \vdots \\[8pt]
     w_{n1} & w_{n2} & \cdots & w_{nm} 
    \end{bmatrix} = \begin{bmatrix}
        {u^{(1)}}^t\\[5pt]
        {u^{(2)}}^t\\[5pt]
        \vdots\\[5pt]
        {u^{(n)}}^t
    \end{bmatrix}.
\end{equation*}

\subsection{Definition of  the positive constant \texorpdfstring{$\mu$}{mu(x,y)} minimizing the learning error \texorpdfstring{$E(W)$}{E(W)}}
\label{subsec:defest}

\noindent We relate the systems $(S_1): L \Box_{\min}^{\max} u = b^{(1)},  (S_2): L \Box_{\min}^{\max} u=b^{(2)}, \dots, (S_n): L \Box_{\min}^{\max}u= b^{(n)}$, associated to the training data to the learning error $E(W)$, see (\ref{eq:ew}) by the following useful result:
\begin{lemma}\label{lemma:7egalitechelou}
 For all matrices $W$ of size $(n , m)$, we have: 
\begin{equation}\label{rel:wtos1s2sn}
E(W) = \max_{1 \leq i \leq N}\, \Vert y^{(i)} - W \Box_{\min}^{\max} x^{(i)}\Vert = 
\max_{1 \leq k \leq n}\, \Vert b^{(k)}  - L \Box_{\min}^{\max} u^{(k)}\Vert,
\end{equation}
where $u^{(k)}$ is the column vector corresponding to the transpose of the $k$-th row of the matrix $W$.
\end{lemma}
\begin{proof}For all   $1 \leq i \leq N$, we have:
\begin{align}
   \Vert y^{(i)} - W \Box_{\min}^{\max} x^{(i)}\Vert  &  =  \max_{1 \leq k \leq n}\,\mid   y^{(i)}_k - 
   \max_{1 \leq j \leq m}\min(w_{kj} , x^{(i)}_j)\mid   \nonumber\\
   & = \max_{1 \leq k \leq n}\,\mid   b^{(k)}_i  - 
   \max_{1 \leq j \leq m}\min(u^{(k)}_j , l_{ij})\mid.   \nonumber\end{align} 
For all   $1 \leq k \leq n$, we have:
\begin{align}
   \Vert b^{(k)}  - L \Box_{\min}^{\max} u^{(k)}\Vert &  =  \max_{1 \leq i \leq N}\,\mid   b^{(k)}_i  -
    \max_{1 \leq j \leq m}\min(l_{ij} , u^{(k)}_j)\mid.
      \nonumber \end{align}
Finally, we get:
\begin{align}
 \max_{1 \leq i \leq N}\,  \Vert y^{(i)} - W \Box_{\min}^{\max} x^{(i)}\Vert  & =  \max_{1 \leq i \leq N}\, 
\max_{1 \leq k \leq n}\,\mid   b^{(k)}_i  - 
   \max_{1 \leq j \leq m}\min(u^{(k)}_j , l_{ij})\mid   \nonumber\\
   &= \max_{1 \leq k \leq n}\, \max_{1 \leq i \leq N}\, 
 \mid   b^{(k)}_i  - 
   \max_{1 \leq j \leq m}\min(l_{ij} , u^{(k)}_j)\mid   \nonumber\\ 
   & =    \max_{1 \leq k \leq n}\,\Vert b^{(k)}  - L \Box_{\min}^{\max} u^{(k)}\Vert.\nonumber\end{align}
\end{proof}

We remark that for $1 \leq k \leq n$, we have:
\begin{itemize}
    \item If the system $(S_k)$ is consistent,   the Chebyshev distance associated to its second member $b^{(k)}$ ,  see (Definition~\ref{def:chebyshevdist}), is equal to zero i.e., $\Delta(L, b^{(k)}) = 0$, so obviously, we have:
$$\Vert b^{(k)}  - L \Box_{\min}^{\max} u^{(k)}\Vert \geq \Delta(L, b^{(k)}) = 0,$$
\noindent where $u^{(k)}$ is the transpose of the $k$-th row of $W$.
\item If the system $(S_k)$ is inconsistent, we note that the system 
formed by the matrix $L$ and the vector $L \Box_{\min}^{\max} u^{(k)}$ as second member is consistent (one of its solution is $u^{(k)}$). By definition of the  Chebyshev distance $\Delta(L, b^{(k)})$,  we have:
 $$\Vert b^{(k)}  - L \Box_{\min}^{\max} u^{(k)}\Vert \geq \Delta(L, b^{(k)}) >  0, \quad \text{ see (Definition  \ref{def:chebyshevdist}).}$$
\end{itemize}

\noindent
These remarks justify the introduction of the following definition: 
\begin{definition}\label{def:mu}
The positive constant $\mu$ minimizing the learning error $E(W)$, see (\ref{eq:muequalminW}), according to the training data is:
\begin{equation}\label{def:mu:error}
\mu := \max_{1 \leq k \leq n}\, \Delta(L, b^{(k)}).
\end{equation}
\end{definition}

\noindent
From (\ref{rel:wtos1s2sn}),  (\ref{def:mu:error}) and the above remarks, we immediately justify that $\mu$ minimizes the learning error $E(W)$:
\begin{proposition}\label{prop:ewsupmu}
For all matrix $W$   of size $(n , m)$, we have:
\begin{equation}\label{Dict49}
E(W)= \max_{1 \leq i \leq N}\, \Vert y^{(i)} - W \Box_{\min}^{\max} x^{(i)}\Vert  \geq \mu.
 \end{equation}
\end{proposition}
\noindent We observe that having a weight matrix $W$ that perfectly relates the input data to the output data i.e. $E(W)=0$, implies having $\mu=0$. In fact, we have:
\begin{proposition}\label{prop:wequivmuzero}
\begin{equation}\label{Dict415}
\exists W \in [0 , 1]^{n \times m} , \text{ s.t. }\forall i \in\{1 , 2 , \cdots , N\}, W \Box_{\min}^{\max} x^{(i)} = y^{(i)}
\,\Longleftrightarrow \,
\mu = 0.
\end{equation}
\end{proposition}\label{eq:equivalencewtrainingmuzero}
\begin{proof}\mbox{}\\
$\Longrightarrow\,$  follows  from (\ref{Dict49}).

\noindent
$\Longleftarrow \,$ If $\mu =  \max_{1 \leq k \leq n}\, \Delta(L, b^{(k)}) = 0$, then  all the systems $(S_1),(S_2),\dots,(S_n)$ are consistent. 

\noindent
For all $1 \leq k \leq n$, let $u^{(k)}\in [0 , 1]^{m \times 1}$ be a solution of the system $(S_k)$. Denote by $W 
\in   [0 , 1]^{n \times m}$ the matrix whose rows are the transpose of the chosen column-vectors $(u^{(1)}, u^{(2)}  \dots , u^{(n)})$ (see (\ref{Dict477})). By (Lemma \ref{lemma:7egalitechelou}), we have:
$$E(W) = \max_{1 \leq i \leq N}\, \Vert y^{(i)} - W \Box_{\min}^{\max} x^{(i)}\Vert = 
\max_{1 \leq k \leq n}\, \Vert b^{(k)}  - L \Box_{\min}^{\max} u^{(k)}\Vert = 0 = \mu,$$
i.e., $\forall i \in\{1 , 2 , \cdots , N\}, W \Box_{\min}^{\max} x^{(i)} = y^{(i)}$.
 \end{proof}

\subsection{Method for learning approximate weight matrices}
\label{subsec:method}
\noindent
In the proof of (Proposition \ref{prop:wequivmuzero}), when    $ \mu = 0$, we have shown how to construct a weight matrix of the  training data. We extend this construction to the general case where $\mu \geq 0$, i.e., we   give a method for constructing an approximate weight  matrix $W$ in the following sense:
\begin{equation}\label{eq:muwbonapprox}
    \max_{1 \leq i \leq N}\, \Vert y^{(i)} - W \Box_{\min}^{\max} x^{(i)}\Vert  = \mu.
\end{equation}

By extending the method developed in (Subsection \ref{subsec:relatingchelou}) to the case of $n$ systems $(S_1),(S_2),\dots,(S_n)$, we  construct such a matrix $W$.

\begin{method}\label{method:defou}
Let $W$ be a matrix defined row by row, which satisfies the following conditions:
\begin{itemize}
    \item If the system $(S_k)$ is consistent, we define the $k$-th row of $W$  as the transpose of a solution $u^{(k)}$ of the system  $(S_k)$. For instance, its greatest solution $L^t \Box_{\rightarrow_G}^{\min} b^{(k)}$. With this choice, we have:
$$\Vert b^{(k)}  - L \Box_{\min}^{\max} u^{(k)}\Vert = 0 =  \Delta(L, b^{(k)}).$$

\item If the system $(S_k)$ is inconsistent, we take a Chebyshev approximation $b^{(k),\ast}$ of $b^{(k)}$ (an element of the non-empty set $\mathcal{C}_{b^{(k)}}$, see (\ref{def:chebyshevsetapproxB})). With this choice, we define the $k$-th row of $W$   as the transpose of a   solution  $u^{(k)}$ of the system $L \Box^{\max}_{\min} u = b^{(k),\ast}$, for instance the greatest solution $L^t \Box_{\rightarrow_G}^{\min} b^{(k),\ast}$. With this choice, we have:
$$\Vert b^{(k)}  - L \Box_{\min}^{\max} u^{(k)}\Vert  =   \Delta(L, b^{(k)}).$$
\end{itemize}

\noindent Thus, any    matrix $W$ constructed row by row with the above procedure will satisfy  (\ref{eq:muwbonapprox}).
\end{method}

We remind that the set ${\cal A}$, see  (\ref{eq:learning:setA}), is the set formed by the matrices verifying (\ref{eq:muwbonapprox}). From (Method \ref{method:defou}), we have:
\begin{proposition}\label{eq:anonvide}
  The set  $\mathcal{A}$ is non-empty. 
\end{proposition}
\begin{proof}
\noindent For $1 \leq k \leq n$, we choose an approximate solution (or solution) $u^{(k)}$ of the system $(S_k)$. \\Let $W \in [0,1]^{n\times m}$ be the matrix defined by:
\[W =   \begin{bmatrix}
        {u^{(1)}}^t\\[5pt]
        {u^{(2)}}^t\\[5pt]
        \vdots\\[5pt]
        {u^{(n)}}^t
    \end{bmatrix}.\]
\noindent From (Lemma \ref{lemma:7egalitechelou}),   (Definition \ref{def:approxsol}) and (Definition \ref{def:mu}), we deduce:
\begin{equation*}
E(W)=    \max_{1 \leq i \leq N}\, \Vert y^{(i)} - W \Box_{\min}^{\max} x^{(i)}\Vert = 
\max_{1 \leq k \leq n}\, \Vert b^{(k)}  - L \Box_{\min}^{\max} u^{(k)}  \Vert = \max_{1 \leq k \leq n}\, \Delta(L, b^{(k)})=\mu. 
\end{equation*}
\noindent Thus $W \in \mathcal{A}$.

\end{proof}
We deduce our main result, i.e., the equality (\ref{eq:muequalminW}) holds:
\begin{corollary}\label{eq:muegalmin}
\[ \mu = \min_{W \in {[0,1]^{n \times m}}} E(W).\]
\end{corollary}
\begin{proof}
This follows from (Proposition \ref{prop:ewsupmu}) and (Proposition \ref{eq:anonvide}).    
\end{proof}

\noindent In what follows, we  illustrate our constructions.

\subsection{Examples}
\label{subsec:example}
The following examples illustrate the learning paradigm. In the first example we have $\mu > 0$, while in the second example, which was introduced by Pedrycz in \cite{pedrycz1991neurocomputations}, we have $\mu=0$. 

\begin{example}\label{subsec:ex1}

\noindent Let us consider the following training data:
\begin{table}[H]
    \centering
    \begin{tabular}{c|c}
       $x^{(1)} =(0.7, 0.4, 0.4)^t$  & $y^{(1)} =(0.7, 0.1, 0.3)^t$  \\
       $x^{(2)} =(1.0, 0.2, 0.5)^t$  & $y^{(2)} =(1.0, 0.7, 0.0)^t$  \\ 
    \end{tabular}
    \caption{Training data of the  example. We have $N=2, m = 3$ and $n=3$.}
    \label{tab:learning:ex2}
\end{table}
\noindent  We have $L = \begin{bmatrix}
        0.7 & 0.4 & 0.4 \\
        1.0 & 0.2 & 0.5
    \end{bmatrix}$, $b^{(1)} = \begin{bmatrix}
        0.7\\
        1.0
    \end{bmatrix}$, $b^{(2)} = \begin{bmatrix}
        0.1\\
        0.7
    \end{bmatrix}$ and $b^{(3)} = \begin{bmatrix}
        0.3\\
        0.0
    \end{bmatrix}$.
\noindent We form three systems $(S_1), (S_2)$ and $(S_3)$:
\begin{align*}
    (S_1):  L \Box_{\min}^{\max} u_1 = b^{(1)}, \\
    (S_2):  L \Box_{\min}^{\max} u_2 = b^{(2)},\\
    (S_3):  L \Box_{\min}^{\max} u_3 = b^{(3)}. 
\end{align*}

\noindent $\bullet$ The system $(S_1)$ is consistent because the Chebyshev distance associated to its second member is~$\Delta(L,b^{(1)}) =~0$.  It has  $\begin{bmatrix}1\\1 \\ 1 \end{bmatrix}$ as greatest solution, and it has a unique minimal solution $\begin{bmatrix}1 \\ 0\\ 0\end{bmatrix}$ computed using the method of \cite{matusiewicz2013increasing}.\\
\noindent $\bullet$ The system $(S_2)$ is inconsistent because the Chebyshev distance associated to its second member is $\Delta(L,b^{(2)}) = 0.3$.  We get $\eta = \begin{bmatrix} 0.4 \\ 1 \\ 1\end{bmatrix}$ and the greatest Chebyshev approximation of $b^{(2)}$ is: $\begin{bmatrix}
    0.4\\
0.5
\end{bmatrix}$ because $L \Box_{\min}^{\max} \eta = \begin{bmatrix}
    0.4\\
0.5
\end{bmatrix}$. The vectors $\begin{bmatrix}0.4 \\0.0\\0.0\end{bmatrix}$ and $\begin{bmatrix}0.0 \\0.0\\0.4\end{bmatrix}$ are solutions of the system of inequalities $\underline{b^{(2)}}(\Delta(L,b^{(2)})) \leq L  \Box_{\min}^{\max} x$ and lower than $\eta$. We have $L  \Box_{\min}^{\max} \begin{bmatrix}0.4 \\0.0\\0.0\end{bmatrix} = L  \Box_{\min}^{\max} \begin{bmatrix}0.0\\0.0\\0.4\end{bmatrix} = \begin{bmatrix}
    0.4\\
    0.4
\end{bmatrix}$, therefore, from (Corollary \ref{corollary:ctitlde}), we have a unique minimal Chebyshev approximation of $b^{(2)}$ which is  $\begin{bmatrix}0.4\\0.4\end{bmatrix}$. We use the greatest  Chebyshev approximation. The system $(S'_2): \begin{bmatrix}
        0.4\\
        0.5
    \end{bmatrix} = \begin{bmatrix}
        0.7 & 0.4 & 0.4 \\
        1.0 & 0.2 & 0.5
    \end{bmatrix} \Box_{\min}^{\max} u'_2$ is consistent and it has  $\begin{bmatrix}0.4\\ 1\\ 1\end{bmatrix}$ as greatest solution and one unique minimal solution  $\begin{bmatrix}0\\ 0\\ 0.5\end{bmatrix}$.\\
\noindent $\bullet$ The system $(S_3)$ is inconsistent because the Chebyshev distance associated to its second member is $\Delta(L, b^{(3)}) = 0.15$.  We use the greatest Chebyshev approximation of $b^{(3)}$: $\begin{bmatrix}
    0.15\\
0.15
\end{bmatrix}$. The system $(S'_3): \begin{bmatrix}
        0.15\\
        0.15
    \end{bmatrix} = \begin{bmatrix}
        0.7 & 0.4 & 0.4 \\
        1.0 & 0.2 & 0.5
    \end{bmatrix} \Box_{\min}^{\max} u'_3$ is consistent and it has  $\begin{bmatrix}0.15\\ 0.15\\ 0.15 \end{bmatrix}$ as greatest solution and  three  minimal solutions $\begin{bmatrix}0.15 \\ 0\\ 0\end{bmatrix}$, $\begin{bmatrix}0 \\ 0.15\\ 0\end{bmatrix}$ and $\begin{bmatrix}0 \\ 0\\ 0.15\end{bmatrix}$.

\noindent As, we have $\Delta(L, b^{(1)}) = 0$,   $\Delta(L, b^{(2)}) = 0.3$ and $\Delta(L, b^{(3)}) =  0.15$, we have $\mu = 0.3$. 

\noindent From the solutions of $(S_1)$, $(S'_2)$ and $(S'_3)$, we can construct an approximate weight matrix $W$ row by row. For instance, $W  = \begin{bmatrix}
    1 & 0 & 0.2\\
    0.2 & 1 & 0.5 \\
    0.15 & 0.15 & 0.0
\end{bmatrix}$ where $\begin{bmatrix}1 \\ 0 \\ 0.2 \end{bmatrix}$ is a solution of $(S_1)$,  $\begin{bmatrix}0.2 \\ 1.0 \\ 0.5 \end{bmatrix}$ is a solution of $(S'_2)$ and  $\begin{bmatrix}0.15 \\ 0.15 \\ 0\end{bmatrix}$ is a solution of $(S'_3)$.   From the training data, we observe that:
\begin{align*}
    W \Box_{\min}^{\max} x^{(1)} = \begin{bmatrix}0.7\\ 0.4\\ 0.15\end{bmatrix} \text{ and } \Vert \begin{bmatrix}0.7\\ 0.4\\ 0.15\end{bmatrix} - y^{(1)} \Vert = 0.3 = \mu, \nonumber\\
    W \Box_{\min}^{\max} x^{(2)} = \begin{bmatrix}1\\ 0.5\\ 0.15 \end{bmatrix} \text{ and } \Vert \begin{bmatrix}1\\ 0.5\\ 0.15\end{bmatrix} - y^{(2)} \Vert = 0.2 < \mu. \nonumber
\end{align*}
\end{example}

\begin{example}

In \cite{pedrycz1991neurocomputations}, Pedrycz learns a weight matrix according to the following training data:
\begin{table}[H]
    \centering
    \begin{tabular}{c|c}
       $x^{(1)} = (0.3, 1.0, 0.5, 0.2)^t$  & $y^{(1)} = (0.7, 0.5, 0.6)^t$  \\
       $x^{(2)} = (0.1, 1.0, 1.0, 0.5)^t$  & $y^{(2)} = (0.7, 1.0, 0.6)^t$  \\ 
       $x^{(3)} = (0.5, 0.7, 0.2, 1.0)^t$  & $y^{(3)} = (0.7, 0.7, 0.6)^t$  \\ 
       $x^{(4)} = (1.0, 0.7, 0.5, 0.3)^t$  & $y^{(4)} = (1.0, 0.5, 0.6)^t$  \\ 
    \end{tabular}
    \caption{Training data used in \cite{pedrycz1991neurocomputations}. We have $N=4, m = 4$ and $n=3$.}
    \label{tab:learning:ex1}
\end{table}

\noindent We put $L=\begin{bmatrix}
        0.3 & 1.0 & 0.5 & 0.2 \\
        0.1 & 1.0 & 1.0 & 0.5 \\
        0.5 & 0.7 & 0.2 & 1.0 \\
        1.0 & 0.7 & 0.5 & 0.3 
    \end{bmatrix}$, $b^{(1)}=\begin{bmatrix}
        0.7 \\
        0.7 \\
        0.7 \\
        1.0
    \end{bmatrix}$,$b^{(2)}= \begin{bmatrix}
        0.5 \\
        1.0 \\
        0.7 \\
        0.5
    \end{bmatrix}$ and $b^{(3)}= \begin{bmatrix}
        0.6 \\
        0.6 \\
        0.6 \\
        0.6
    \end{bmatrix}$. We form  three systems $
    (S_1):   L\Box_{\min}^{\max} u_1 = b^{(1)}$, $
    (S_2):   L \Box_{\min}^{\max} u_2 = b^{(2)}$ and  $(S_3):  L \Box_{\min}^{\max} u_3 =b^{(3)}$. 

\noindent $\bullet$ The system $(S_1)$ is consistent because the Chebyshev distance associated to its second member is~$\Delta(L, b^{(1)}) =~0$. Its greatest solution is $\begin{bmatrix}
            1.0 \\
            0.7 \\
            0.7\\
            0.7
        \end{bmatrix}$ and it has a unique minimal solution~$\begin{bmatrix}
            1.0 \\
            0.7 \\
            0.0 \\
            0.0
        \end{bmatrix}$.\\
\noindent $\bullet$ The system $(S_2)$ is consistent because the Chebyshev distance associated to its second member is $\Delta(L, b^{(2)}) = 0$.  Its greatest solution is $\begin{bmatrix}0.5\\ 0.5\\ 1.0 \\  0.7\end{bmatrix}$ and it has a unique minimal solution~$\begin{bmatrix}
        0.0\\ 0.0\\ 1.0\\ 0.7
    \end{bmatrix}$. \\
\noindent $\bullet$ The system $(S_3)$ is consistent because the Chebyshev distance associated to its second member is  $\Delta(L, b^{(3)}) =~0$.       Its greatest solution is $\begin{bmatrix}0.6 \\ 0.6 \\ 0.6\\ 0.6\end{bmatrix}$ and it has a unique minimal solution~$\begin{bmatrix}
            0.0 \\
            0.6 \\
            0.0 \\
            0.0
        \end{bmatrix}$.\\
\noindent As,  $\Delta(L, b^{(1)}) =   \Delta(L, b^{(2)}) = \Delta(L, b^{(3)}) =  0.0$ we have: $\mu = 0.0$. 

\noindent Therefore, in the set of matrices $\mathcal{A}$, see (\ref{eq:learning:setA}), each of the weight matrices $W$ has three rows constructed  from the minimal solution and the greatest solution of each of the three systems $(S_1), (S_2)$ and $(S_3)$:
\begin{equation*}
  \begin{bmatrix}
            1.0 \\
            0.7 \\
            0.0\\
            0.0
        \end{bmatrix} \leq u_1 \leq \begin{bmatrix}
            1.0 \\
            0.7 \\
            0.7 \\
            0.7
        \end{bmatrix}, \,   \begin{bmatrix}
            0.0 \\
            0.0 \\
            1.0\\
            0.7
        \end{bmatrix} \leq u_2 \leq \begin{bmatrix}
            0.5 \\
            0.5 \\
            1.0 \\
            0.7
        \end{bmatrix} \text{ and } \begin{bmatrix}
            0.0 \\
            0.6 \\
            0.0\\
            0.0
        \end{bmatrix} \leq u_3 \leq \begin{bmatrix}
            0.6 \\
            0.6 \\
            0.6 \\
            0.6
        \end{bmatrix}.
\end{equation*}
    
\noindent Let us consider for example the weight matrix $W = \begin{bmatrix}
    1.0 & 0.7 & 0.3 & 0.3 \\
    0.4 & 0.4 & 1.0 & 0.7 \\
    0.1 & 0.6 & 0.2 & 0.2 
\end{bmatrix} \in \mathcal{A}$. One can check from the training data that:
\begin{gather*}
    W \Box_{\min}^{\max} x^{(1)} = y^{(1)}, \\
    W \Box_{\min}^{\max} x^{(2)} = y^{(2)},\\ 
    W \Box_{\min}^{\max} x^{(3)} = y^{(3)},\\
    W \Box_{\min}^{\max} x^{(4)} = y^{(4)}.
\end{gather*}

\end{example}

\section{Application}
\label{sec:applications}
In what follows, we study an application of our results: how to approximately learn the rule parameters of a possibilistic rule-based system. 
Recently, Dubois and Prade have emphasized the  development of  possibilistic learning methods that would be consistent with if-then rule-based reasoning \cite{dubois2020possibilistic}. For this purpose, the author of \cite{baaj2022learning} introduced a system of  $\min-\max$ fuzzy relational equations for learning the rule parameters of a possibilistic rule-based system according to a training datum: 
\[(\Sigma): Y = \Gamma \Box_{\max}^{\min} X, \]
\noindent where $\Box_{\max}^{\min}$ is the matrix product
 which takes $\max$ as the product and $\min$ as the addition. In the equation system  $(\Sigma)$, the second member $Y$ describes an output possibility distribution, the matrix $\Gamma$ contains the  possibility degrees  of the rule premises and $X$ is an unknown vector containing the rule parameters. 
 If the  system $(\Sigma)$ is inconsistent, e.g., due to poor training data, an approximate solution is desirable. The general method that we introduced for obtaining approximate solutions of a system of $\max-\min$ fuzzy relational equations can be applied to the case of a system of $\min-\max$ fuzzy relational equations
 such as $(\Sigma)$.

\noindent In what follows, we show how to switch from a system of $\min-\max$ fuzzy relational equations such as $(\Sigma)$ to a system of $\max-\min$ fuzzy relational equations  and vice versa. We introduce analogous tools for a system of $\min-\max$ fuzzy relational equations to those already introduced for a system of $\max-\min$ fuzzy relational equations and we show their  correspondences in (Table \ref{tab:correspondences}).
From these results, we propose a method for finding approximate solutions of the rule parameters of possibilistic rule-based system when we have multiple training data. 

\subsection{Switching from a system of \texorpdfstring{$\max-\min$}{max-min} fuzzy relational equations to a system of \texorpdfstring{$\min-\max$}{min-max} fuzzy relational equations (and vice versa)}
\label{subsec:switch}

In this subsection, we use the following notation:
\begin{notation}
    To any matrix $A = [a_{ij}]$, we associate the matrix  $A^{\circ} = [1 - a_{ij}]$ and we have $(A^{\circ})^{\circ} = A$. 
\end{notation}

\noindent Let $A$ and $B$ be matrices  of respective size $(n,m)$ and $(m,p)$, the transformation $A \mapsto A^{\circ}$ switches the two matrix products $\Box_{\max}^{\min}$ and $\Box_{\min}^{\max}$ in the following sense: 
\begin{equation}\label{eq:Atensetruc}
    (A \Box_{\max}^{\min} B)^{\circ} = A^{\circ} \Box_{\min}^{\max} B^{\circ} \text{ and }  (A \Box_{\min}^{\max} B)^{\circ}  = A^{\circ} \Box_{\max}^{\min}B^{\circ}.
\end{equation}
\noindent This transformation establishes that the study of systems of $\max-\min$ fuzzy relational equations is equivalent to the study of systems of $\min-\max$  fuzzy relational equations in a precise sense that we will develop in what follows and summarize in (Table \ref{tab:correspondences}).\\ 
Let us remark that the transformation $t \mapsto 1 - t$ switches the Gödel product, see (\ref{eq:implication:godel}),  to  the $\epsilon$-product  defined by:
\[ x \epsilon  y = \left\{\begin{array}{rrl}
y & \text{if}  & x < y\\
0 & \text{if}  & x \geq  y \\\end{array}\right. \text{ in } [0, 1].\]
Therefore, we deduce that the matrix product $\Box_{\rightarrow_G}^{\min}$ is switched to the matrix product $\Box_{\epsilon}^{\max}$ where we take the $\epsilon$-product as product and $\max$ as addition.\\

Let:
\begin{itemize}
    \item $A \Box_{\min}^{\max} x = b$ be a system of $\max-\min$ fuzzy relational equations,
    \item $G \Box_{\max}^{\min} x = d$ be a system of $\min-\max$ fuzzy relational equations.
\end{itemize}
\noindent In (Table \ref{tab:correspondences}), for a system $G \Box_{\max}^{\min} x = d$, we introduce analogous tools (second column) to those already introduced for a system $A \Box_{\min}^{\max} x = b$ (first column). The last column shows how to relate the tools of the two systems iff
\begin{equation}
    G = A^\circ\text{ and }d = b^\circ.
\end{equation}

\begin{table}[H]
\footnotesize
\centering
\begin{tabular}{c|c|c|c|}
\hhline{~---}
& \multicolumn{1}{c|}{System: $A \Box_{\min}^{\max} x = b$} & \multicolumn{1}{c|}{System: $G \Box_{\max}^{\min} x = d$} & 
\multicolumn{1}{c|}{
\cellcolor[HTML]{C0C0C0}
\begin{tabular}[c]{@{}l@{}}
Relation \\
 iff $G = A^\circ$ and $d= b^\circ$
\end{tabular}}\\\hline

\multicolumn{1}{|l|}{\begin{tabular}[c]{@{}l@{}}Set of solutions\\ $\quad$\end{tabular}}  &{${\cal {S}}(A,b)$}& {${\cal {S}}(G,d)$}& {\cellcolor[HTML]{C0C0C0}${\cal {S}}(G,d)={\cal {S}}(A,b)^{\circ}$} \\ \hline
\multicolumn{1}{|l|}{
\begin{tabular}[c]{@{}l@{}}Potential greatest/lowest\\ solution\end{tabular}}   &   \begin{tabular}[c]{@{}l@{}} $e=A^t\Box_{\rightarrow_G}^{\min} b$ \\ (greatest solution) \end{tabular}     & \begin{tabular}[c]{@{}l@{}} $ r=G^t\Box_{\epsilon}^{\max} d $\\ (lowest solution) \end{tabular}        &  \cellcolor[HTML]{C0C0C0}$r= e^\circ$  \\ \hline
\multicolumn{1}{|l|}{\begin{tabular}[c]{@{}l@{}}Application computing \\ the matrix product of\\ the system matrix and\\a given vector in $[0,1]^{m\times 1}$\end{tabular}}&{
\begin{tabular}[c]{@{}l@{}}
$\theta: [0,1]^{m\times 1}\rightarrow[0,1]^{n\times 1}$\\ 
$\,\,\,\,\, : x \mapsto A \Box_{\min}^{\max} x$
\end{tabular}

}& {
\begin{tabular}[c]{@{}l@{}}
$\psi: [0,1]^{m\times 1}\rightarrow[0,1]^{n\times 1}$ \\
$\,\,\,\,\,\, : x \mapsto G \Box_{\max}^{\min} x$
\end{tabular}
}&  {
\cellcolor[HTML]{C0C0C0}$\psi(x)=\theta(x^{\circ})^{\circ}$
}\\ \hline 
\multicolumn{1}{|l|}{\begin{tabular}[c]{@{}l@{}}Set of second members\\  of the consistent systems \\ defined with the matrix\end{tabular}}                                        &    {${\cal {C}}= \{ \theta(x) \mid x \in [0,1]^{m\times 1} \}$}& {${\cal {T}}= \{ \psi(x) \mid x \in [0,1]^{m\times 1} \}$}& {
\cellcolor[HTML]{C0C0C0}${\cal {T}}={\cal {C}}^{\circ}$
}\\ \hline 
\multicolumn{1}{|l|}{\begin{tabular}[c]{@{}l@{}}Application for checking\\ 
if a system defined with\\ the matrix  and a given \\vector in $[0,1]^{n \times 1}$
as \\second member is a\\
consistent system\end{tabular}}       &   
\begin{tabular}[c]{@{}l@{}}
$F : [0,1]^{n \times 1} \rightarrow [0,1]^{n \times 1} $\\
 $\,\,\,\,\,\,\,\,\,c \mapsto A \, \Box_{\min}^{\max} (A^t\, \Box_{\rightarrow_G}^{\min} c)$ 
 \end{tabular}
 
 & 
 \begin{tabular}[c]{@{}l@{}}
$U : [0,1]^{n \times 1} \rightarrow [0,1]^{n \times 1} $\\
 $\,\,\,\,\,\,\,\,\,c \mapsto G\Box_{\max}^{\min} (G^t\, \Box_{\epsilon}^{\max} c)$ 
 \end{tabular}
        & \cellcolor[HTML]{C0C0C0} $U(c)=F(c^{\circ})^{\circ}$ \\ \hline

\multicolumn{1}{|l|}{\begin{tabular}[c]{@{}l@{}}Chebyshev distance\\ associated  to \\the second member\end{tabular}}                                                                & $\Delta=\Delta(A,b)$        & $\nabla=\nabla(G,d)$        &                    \cellcolor[HTML]{C0C0C0}            $\nabla(G,d)=\Delta(A,b)$        \\ \hline
\multicolumn{1}{|l|}{\begin{tabular}[c]{@{}l@{}}Set of Chebyshev\\ approximations of\\ the second member\end{tabular}}  &                                                      $\mathcal{C}_b$      &   $\mathcal{T}_{d}$               &                 \cellcolor[HTML]{C0C0C0}        $\mathcal{T}_d = \mathcal{C}_b^\circ$               \\ \hline
\multicolumn{1}{|l|}{\begin{tabular}[c]{@{}l@{}}Extremal Chebyshev\\ approximations \\ of the second member\end{tabular}}  & 
\begin{tabular}[c]{@{}l@{}}greatest: $F(\overline{b}(\Delta))$\\
minimal approx. set: $\mathcal{C}_{b,\min}$
\end{tabular}   &  
\begin{tabular}[c]{@{}l@{}}lowest: $U(\underline{d}(\nabla))$\\
 maximal approx. set: $\mathcal{T}_{d,\max}$
\end{tabular}
        &                 \cellcolor[HTML]{C0C0C0}  
        \begin{tabular}[c]{@{}l@{}}
        $U(\underline{d}(\nabla))={F(\overline{b}(\Delta))}^\circ$
        \\
        $\mathcal{T}_{d,\max} = \mathcal{C}_{b,\min}^\circ$     \end{tabular}          \\ \hline
\multicolumn{1}{|l|}{\begin{tabular}[c]{@{}l@{}}Approximate solutions set\vspace{1em}\end{tabular}}  & 
\begin{tabular}[c]{@{}l@{}}$\Lambda_b$ 
\end{tabular}   &  
\begin{tabular}[c]{@{}l@{}}$\Upsilon_d$
\end{tabular}
        &                 \cellcolor[HTML]{C0C0C0}  
        \begin{tabular}[c]{@{}l@{}}
        $\Upsilon_d = \Lambda_b^\circ$
             \end{tabular}          \\ \hline

\multicolumn{1}{|l|}{\begin{tabular}[c]{@{}l@{}}Extremal approximate \\ solutions\end{tabular}}  & 
\begin{tabular}[c]{@{}l@{}}greatest: $\eta=A^t \Box_{\rightarrow_G}^{\min} F(\overline{b}(\Delta))$\\
a min. approx. sol. set: ${\Lambda}_{b,\min}$
\end{tabular}   &  
\begin{tabular}[c]{@{}l@{}}lowest: $\nu=G^t \Box_{\epsilon}^{\max} U(\underline{d}(\nabla))$\\
 a max. approx. sol. set: $\Upsilon_{d,\max}$
\end{tabular}
        &                 \cellcolor[HTML]{C0C0C0}  
        \begin{tabular}[c]{@{}l@{}}
        $\nu = \eta^\circ$
        \\
        $\Upsilon_{d,\max} = \Lambda_{b,\min}^\circ$     \end{tabular}          \\ \hline
\end{tabular}
\vspace{1em}
\caption{Tools of the systems $A \Box_{\min}^{\max} x = b$ and $G \Box_{\max}^{\min} x = d$ and their relations iff $G = A^\circ$ and $d= b^\circ$.}
\label{tab:correspondences}
\end{table}

\noindent The relations established for the first five rows of  (Table \ref{tab:correspondences}) are justified by the general switch in  (\ref{eq:Atensetruc}). Assuming the relation in the sixth row is established, the remaining rows are also justified by the general switch in (\ref{eq:Atensetruc}). 

In the following, we define the Chebyshev distance associated to the second member $d$ of the system $G \Box_{\max}^{\min} x = d$, denoted $\nabla(G,d)$ and we prove the last four relations.  

\begin{definition}\label{def:chebdistminmax}
The Chebyshev distance associated to the second member $d$ of the system $G \Box_{\max}^{\min} x = d$ is:
\begin{align*}
    \nabla(G,d) &= \inf_{ c \in \mathcal{T}} \Vert d - c \Vert.
\end{align*}
\end{definition}

We remark that if  $d=b^\circ$, then for all $c \in [0,1]^{n \times 1}$ we have $\Vert d - c \Vert = \Vert b - c^\circ \Vert.$ From this property, we deduce:
\begin{proposition}
If $G = A^\circ$ and $d=b^\circ$, then the Chebyshev distance associated to the second member $d$ of the system $G \Box_{\max}^{\min} x = d$ is equal to  the Chebyshev distance associated to the second member $b$ of the system $A \Box_{\min}^{\max} x = b$: 
    \begin{equation}\label{equaldeltas}
        \nabla(G,d) = \Delta(A,b).
    \end{equation} 
\end{proposition}
\begin{proof}
\begin{align*}
    \nabla(G,d) &=  \inf_{ c \in \mathcal{T}} \Vert d - c \Vert\\
    &= \inf_{ c \in \mathcal{T}} \Vert b - c^\circ \Vert\\
    &= \inf_{ c'\in \mathcal{C}} \Vert b - c' \Vert \quad\quad \text{(because $\mathcal{C}= \mathcal{T}^\circ$)}\\
    &= \Delta(A,b) \quad\quad\quad \text{(see (Definition \ref{def:chebyshevdist}))}\\
\end{align*}
\end{proof}
\noindent The equality $\nabla(G,d) = \Delta(A,b)$, allows us to establish analogous properties for $\nabla(G,d)$:

\begin{corollary}
     $\nabla(G,d) = \min_{ c \in \mathcal{T}} \Vert d - c \Vert.$
\end{corollary}
\noindent In order to give an explicit formula for $\nabla(G,d)$, we will use the following lemma:
\begin{lemma}If $G = A^\circ$ and $d=b^\circ$, then we have: 
\begin{equation}
    \forall c \in [0 , 1]^{n \times 1}, \forall \delta\in [0 , 1],\quad U({\underline{c}(\delta)}) \leq {\overline{c}(\delta)}\,\Longleftrightarrow \, {\underline{c'}(\delta)} \leq F({\overline{c'}(\delta)}), 
\end{equation}
\noindent where  
$c' = c^\circ$. 
\end{lemma}
\begin{proof}
This is a consequence of the relation  $\forall c \in [0 , 1]^{n \times 1}, U(c)=F(c^{\circ})^{\circ}$. 
\end{proof}
For a  system of $\min-\max$ fuzzy relational equations, (Theorem 1 of \cite{cuninghame1995residuation})   becomes: 
\begin{corollary}
 $\nabla(G,d) =        \min\{\delta\in [0, 1] \mid U(\underline{d}(\delta)) \leq \overline{d}(\delta) \}$.   
\end{corollary}
For a  system of $\min-\max$ fuzzy relational equations, (Theorem \ref{th:Deltamin}) of this article becomes:
\begin{corollary} Let $G = [g_{ij}]_{1 \leq i \leq n, 1 \leq j \leq m} \in [0,1]^{n\times m}$ be a matrix and $d = [d_i]_{1 \leq i\leq n}$ be a column vector. The Chebyshev distance associated to the second member $d$ of the system $G \Box_{\max}^{\min} x = d$ is:
\begin{equation}\label{eq:nablapremier}
    \nabla = \nabla(G,d) =  \max_{1 \leq i \leq n}\,\nabla_i
\end{equation}
where for $i = 1, 2, \dots n$: 
\begin{equation} \label{eq:Nablamini} 
\nabla_i =  \min_{1 \leq j \leq m}\,\max[ ( g_{ij}-d_i)^+,  \max_{1 \leq k \leq n}\,  \,\sigma_\epsilon\,(d_i, g_{kj}, d_k)]
\end{equation}
\noindent and
\begin{equation}\label{eq:sigmaespi}
    \sigma_\epsilon\,(u,v,w) = \min(\frac{(w-u)^+}{2}, (w - v)^+).  
\end{equation}
\end{corollary}
\begin{proof}
If we set $A = G^\circ$ and $b = d^\circ$,  we deduce (\ref{eq:nablapremier}), (\ref{eq:Nablamini}) and (\ref{eq:sigmaespi})  from the equality $\Delta(A,b)=\nabla(G,d)$, see (\ref{equaldeltas}), (Theorem \ref{th:Deltamin}) and the relation $\sigma_G(x,y,z)=\sigma_\epsilon(u,v,w)$ where $x = u^\circ$, $y= v^\circ$ and $z= w^\circ$. 
\end{proof}

We define the set of Chebyshev approximations of the second member $d$ of the system $G \Box_{\max}^{\min} x = d$:
\begin{equation}
    \mathcal{T}_d = \{ c \in \mathcal{T} \mid \Vert d -c \Vert = \nabla(G,d)   \}.
\end{equation}
\noindent  If $G = A^\circ$ and $d=b^\circ$, then the equality $\mathcal{T}_d = \mathcal{C}^\circ_b$ follows from (\ref{equaldeltas}).\\
From $U(\underline{d}(\nabla))={F(\overline{b}(\Delta))}^\circ$, see (Table \ref{tab:correspondences}), we
 deduce:
\begin{corollary}
 The lowest Chebyshev approximation of the second member $d$ of the system $G \Box_{\max}^{\min} x = d$ is $U(\underline{d}(\nabla))$.
\end{corollary}
\noindent The method for obtaining maximal Chebyshev approximations of the second member $d$ of the system $G \Box_{\max}^{\min} x = d$ is  analogous  to the practical method presented in (Subsection \ref{subsec:minicheb}).
We use the following notation:
\begin{notation}\label{notationwwww}
Let 
    $\{ w^{(1)},w^{(2)},\dots,w^{(h)} \}$  be the set of maximal solutions of the system of inequalities $ G  \Box_{\max}^{\min} x \leq \overline{d}(\nabla)$  such that $\forall i \in \{1,2,\dots,h\}, w^{(i)} \geq \nu=G^t \Box_{\epsilon}^{\max} U(\underline{d}(\nabla))$.

\end{notation}
We have:
\begin{corollary}\label{eq:maxsol} 
\noindent We  put:
\begin{equation*}\label{eq:ttilde}
   \widetilde{\mathcal{T}} = \{ \psi(w^{(1)}),\psi(w^{(2)}),\dots,\psi(w^{(h)}) \} 
\end{equation*}
\noindent and
\begin{equation*}\label{eq:ttildemin}
    (\widetilde{\mathcal{T}})_{\max} = \{ c \in \widetilde{\mathcal{T}} \mid c \text{ is maximal in }\widetilde{\mathcal{T}}\}.
\end{equation*}
 
\noindent Then, we have: \[\widetilde{\mathcal{T}} \subseteq \mathcal{T}_{d} \text{ and } \mathcal{T}_{d,\max} = (\widetilde{\mathcal{T}})_{\max}, \]
\noindent where  $\mathcal{T}_{d,\max}$ is the set  formed by the maximal Chebyshev approximations of the second member $d$ of the system $G \Box_{\max}^{\min} x = d$.
\end{corollary}
\begin{proof}
Let $A = G^\circ$ and $b = d^\circ$.  From the equality $\nabla(G,d)=\Delta(A,b)$ and the general switch in (\ref{eq:Atensetruc}), we have for any $x \in [0,1]^{m \times 1}$:
\[ G \Box^{\min}_{\max} x \leq \overline{d}(\nabla) \Longleftrightarrow \underline{b}(\Delta) \leq A \Box_{\min}^{\max} x^\circ, \]
\[ x \geq \nu \Longleftrightarrow x^\circ \leq \eta=A^t \Box_{\rightarrow_G}^{\min} F(\overline{b}(\Delta)). \]
\noindent From these two equivalences, we deduce:
\[  \{ w^{(1)},w^{(2)},\dots,w^{(h)} \} = \{{v^{(1)}}^\circ,{v^{(2)}}^\circ,\dots,{v^{(h}}^\circ \}    \]
\noindent where the set $\{v^{(1)},v^{(2)},\dots,v^{(h)}  \}$ is defined in (Notation \ref{notations:minimalsol}) for the system $A \Box_{\min}^{\max} x = b$.  \\
\noindent Using the switch (\ref{eq:Atensetruc}), this last equality implies the claims of the  Corollary.
\end{proof}
\begin{corollary}\label{eq:tdmaxnonempty}
The set $\mathcal{T}_{d,\max}$ is non-empty and finite. 
\end{corollary}
\begin{proof}
As $\widetilde{{\cal T}}$ is a finite non-empty ordered set, the set $(\widetilde{\mathcal{T}})_{\max} = \mathcal{T}_{d,\max}$ is non-empty and finite. 
\end{proof}

\noindent We study the approximate solutions set  $\Upsilon_d$ of the system $G \Box_{\max}^{\min} x = d$:
\begin{definition}
The approximate solutions set of the system $G \Box_{\max}^{\min} x = d$ is:
    \begin{equation*}
    \Upsilon_d = \psi^{-1}(\mathcal{T}_{d}) = \{x\in[0,1]^{m\times 1} \mid \psi(x) \in \mathcal{T}_{d}\}.
\end{equation*}
\end{definition}
\noindent If $A = G^\circ$ and $b = d^\circ$, we have $\Upsilon_{d}= \Lambda_b^\circ$. From this, we deduce a particular element of $\Upsilon_d$:
\begin{proposition}
   The lowest approximate solution of the system  $G \Box_{\max}^{\min} x = d$ is $\nu = G^t \Box_{\epsilon}^{\max} U(\underline{d}(\nabla))$.
\end{proposition}
\begin{proof}
This follows from that the fact that, if $A = G^\circ$ and $b = d^\circ$, we have $\nu = \eta^\circ$.    
\end{proof}
  We have a set of maximal approximate solutions $\Upsilon_{d,\max}$ of the system  $G \Box_{\max}^{\min} x = d$  that satisfies:
\begin{equation}\label{eq:satisifactiontruc}
    \Upsilon_{d,\max} \subseteq \Upsilon_d  \text{ and } \mathcal{T}_{d,\max} = \{ \psi(x) \mid x \in \Upsilon_{d,\max} \},
\end{equation} 
which is defined by:
\begin{definition}
$$\Upsilon_{d,\max}= \{ x \in \{ w^{(1)}, w^{(2)},\dots, w^{(h)} \} \mid \psi(x) \in \mathcal{T}_{d,\max} \}, \quad \text{see (Notation \ref{notationwwww})}.$$
\end{definition}
We have: 
\begin{proposition}
With the above definition of $\Upsilon_{d,\max}$, we have $\Upsilon_{d,\max} \subseteq \Upsilon_d  \text{ and } \mathcal{T}_{d,\max} = \{ \psi(x) \mid x \in \Upsilon_{d,\max} \}$. Therefore, the set $\Upsilon_{d,\max}$ is  non-empty and finite.
\end{proposition}
\begin{proof}
 If $A = G^\circ$ and $b = d^\circ$, we have: $\psi(x)=\theta(x^{\circ})^{\circ}$,
 $\Upsilon_{d} = \Lambda_{b}^\circ$ and $\Upsilon_{d,\max} = \Lambda_{b,\min}^\circ$. From these three equalities, we deduce immediately $\Upsilon_{d,\max} \subseteq \Upsilon_d$ and $\mathcal{T}_{d,\max} = \{ \psi(x) \mid x \in \Upsilon_{d,\max} \}$.\\
As we know that the set $\mathcal{T}_{d,\max}$ is non-empty and finite (Corollary \ref{eq:tdmaxnonempty}), we deduce from the equality $\mathcal{T}_{d,\max} = \{ \psi(x) \mid x \in \Upsilon_{d,\max} \}$ that the set $\Upsilon_{d,\max}$ is also non-empty and finite.
\end{proof}

 The structure of the set ${\cal T}_d$ is  described by the following result:
\begin{corollary}\label{corollary:anologth3} For all $c\in [0 , 1]^{n \times 1}$, we have: 
\begin{equation}c \text{ is a Chebyshev approximation of } d \text{ i.e., } c\in {\cal T}_d \,\Longleftrightarrow \, U(c)  = c  \,\,\text{and}\,\,\exists \, c'\in {\cal T}_{d,\max} \,\,
\text{s.t.}\,\, U(\underline{d}(\nabla)) \leq c \leq c'.\end{equation}
\end{corollary}
\begin{proof}
This follows from that the fact that, if $A = G^\circ$ and $b = d^\circ$, we have $\mathcal{T}_d = \mathcal{C}_b^\circ$, $\mathcal{T}_{d,\max} = \mathcal{C}_{b,\min}^\circ$ and $U(\underline{d}(\nabla))={F(\overline{b}(\Delta))}^\circ$; by applying (Theorem \ref{th:3}), we get the   result. 
\end{proof}

\noindent In the following, we illustrate the switch from the system of $\min-\max$ fuzzy relations equations  $(\Sigma)$ of \cite{baaj2022learning} to  its associated system of $\max-\min$ fuzzy relational equations. 

\begin{example}\label{ex:example7poss}
Let us reuse the example in \cite{baaj2022learning}. 
\begin{align*}
(\Sigma): \, \,\, \, Y\, \,\, \, &= \Gamma \Box_{\max}^{\min} X\\
 \begin{bmatrix}
    0.3 \\ 1 \\ 0.3 \\ 0.8 \\ 0.3 \\ 0.7 \\  0.3 \\ 0.7
    \end{bmatrix} &= \begin{bmatrix}
0.1 & 1 & 1 & 1 & 1 & 1\\ 
1 & 1 & 1 & 1 & 1 & 1\\
0.1 & 1 & 1 & 0.8 & 1 & 1\\
1 & 1 & 1 & 0.8 & 1 & 1\\ 
0.1 & 1 & 1 & 1 & 1 & 0.3\\ 
1 & 1 & 1 & 1 & 1 & 0.3\\
0.1 & 1 & 1 & 0.8 & 1 & 0.3\\
1 & 1 & 1 & 0.8 & 1 & 0.3\\
    \end{bmatrix}\Box_{\max}^{\min} \begin{bmatrix}
    s_1 \\
    r_1 \\
    s_2 \\
    r_2 \\
    s_3 \\
    r_3 
    \end{bmatrix}.    
\end{align*}
\noindent where $s_1,r_1,s_2,r_2,s_3$, and $r_3$ are unknown rule parameters. The system $(\Sigma)$ is consistent.  
We have:
\begin{align*}
X=\begin{bmatrix}
    s_1 \\
    r_1 \\
    s_2 \\
    r_2 \\
    s_3 \\
    r_3 
    \end{bmatrix}\text{ is a solution iff} \begin{bmatrix}
0.3 \\
0 \\
0 \\
0 \\
0 \\
0.7
 \end{bmatrix}
 \leq 
 \begin{bmatrix}
    s_1 \\
    r_1 \\
    s_2 \\
    r_2 \\
    s_3 \\
    r_3 
    \end{bmatrix}
    \leq 
 \begin{bmatrix}
 0.3 \\
 1 \\
 1 \\
 0.8\\
 1 \\
 0.7
 \end{bmatrix}.
\end{align*}
\noindent Let $A = \Gamma^\circ$,  $x= X^\circ$ and  $b= Y^\circ$. We have:
\begin{align*}
 \, \,\, \, b\, \,\, \, &= A \Box_{\min}^{\max} x\\
\begin{bmatrix}
   0.7 \\ 0 \\ 0.7 \\ 0.2 \\ 0.7 \\ 0.3 \\  0.7 \\ 0.3
   \end{bmatrix} &= \begin{bmatrix}
0.9 & 0 & 0 & 0 & 0 & 0\\ 
0 & 0 & 0 & 0 & 0 & 0\\
0.9 & 0 & 0 & 0.2 & 0 & 0\\
0 & 0 & 0 & 0.2 & 0 & 0\\ 
0.9 & 0 & 0 & 0 & 0 & 0.7\\ 
0 & 0 & 0 & 0 & 0 & 0.7\\
0.9 & 0 & 0 & 0.2 & 0 & 0.7\\
0 & 0 & 0 & 0.2 & 0 & 0.7\\
   \end{bmatrix}\Box_{\min}^{\max} \begin{bmatrix}
   1 - s_1 \\
   1 - r_1 \\
   1 - s_2 \\
   1 - r_2 \\
   1 - s_3 \\
   1 - r_3 
   \end{bmatrix}.    
\end{align*}  

We have:
\begin{align*}
x=\begin{bmatrix}
    1-s_1 \\
    1-r_1 \\
    1-s_2 \\
    1-r_2 \\
    1-s_3 \\
    1-r_3 
    \end{bmatrix}\text{ is a solution iff} 
  \begin{bmatrix}
 0.7 \\
 0 \\
 0 \\
 0.2\\
 0 \\
 0.3
 \end{bmatrix}
 \leq 
 \begin{bmatrix}
    1-s_1 \\
    1-r_1 \\
    1-s_2 \\
    1-r_2 \\
    1-s_3 \\
    1-r_3 
    \end{bmatrix}
    \leq 
 \begin{bmatrix}
0.7 \\
1 \\
1 \\
1 \\
1 \\
0.3
 \end{bmatrix}.
\end{align*}
\end{example}

\noindent Let  $(\Sigma)$ be an inconsistent system. Using our results, we can obtain approximate solutions of the system $(\Sigma)$ which are solutions of the consistent systems defined by the  matrix of $(\Sigma)$ and a Chebyshev approximation of the  second member of $(\Sigma)$. For obtaining maximal Chebyshev approximations of the second member of the system $(\Sigma)$, we have to use (Corollary  \ref{eq:maxsol}), which requires the solving of a particular system of $\min-\max$ fuzzy relational inequalities. The solving of such a  system can be done by adapting the results of \cite{matusiewicz2013increasing} using an approach similar to the one we used to establish the correspondences (Table \ref{tab:correspondences}).  

\noindent We illustrate how to obtain approximate solutions of the system $(\Sigma)$ when it is inconsistent.
\begin{example}
(continued)
Let us reuse the matrix $\Gamma$ of the previous example and a new second member $Y = \begin{bmatrix}
    0.3 \\ 1 \\ 0.3 \\ 0.8 \\ 0.7 \\ 0.7 \\  0.3 \\ 0.7
    \end{bmatrix}$. In this case, the system $(\Sigma)$ is inconsistent and the Chebyshev distance associated to the second member $Y$ of $(\Sigma)$ is $\nabla = 0.2$. 
    The lowest Chebyshev approximation of $Y$  is denoted $\check{Y}$ and from (Corollary  \ref{eq:maxsol}) we find that there is a unique maximal Chebyshev approximation of $Y$  which is denoted $\hat{Y}$:
    \begin{equation}
        \check{Y} = \begin{bmatrix} 0.5\\ 1\\0.5\\ 0.8\\ 0.5\\ 0.5\\ 0.5\\ 0.5\end{bmatrix} \quad \text{and} \quad \hat{Y} = \begin{bmatrix}
        0.5 \\1\\ 0.5\\ 1\\ 0.5\\ 0.9\\ 0.5\\ 0.9
    \end{bmatrix}.
    \end{equation}

\noindent Some approximate solutions of the system $(\Sigma) : Y = \Gamma \Box_{\max}^{\min} X$ are  the solutions of the system $\check{Y} = \Gamma \Box_{\max}^{\min} X$ i.e.,  $
\begin{bmatrix}
0.5 \\
0 \\
0 \\
0 \\
0 \\
0.5
 \end{bmatrix}
 \leq 
 \begin{bmatrix}
    s_1 \\
    r_1 \\
    s_2 \\
    r_2 \\
    s_3 \\
    r_3 
    \end{bmatrix}
    \leq 
 \begin{bmatrix}
 0.5 \\
 1 \\
 1 \\
 0.8\\
 1 \\
 0.5
 \end{bmatrix}$ and the  solutions of the system $\hat{Y} = \Gamma \Box_{\max}^{\min} X$  i.e, $
\begin{bmatrix}
0.5 \\
0 \\
0 \\
1 \\
0 \\
0.9
 \end{bmatrix}
 \leq 
 \begin{bmatrix}
    s_1 \\
    r_1 \\
    s_2 \\
    r_2 \\
    s_3 \\
    r_3 
    \end{bmatrix}
    \leq 
 \begin{bmatrix}
 0.5 \\
 1 \\
 1 \\
 1\\
 1 \\
 0.9
 \end{bmatrix}$. One can check that $\begin{bmatrix}
 0.5 \\
 1 \\
 1 \\
 1\\
 1 \\
 0.9
 \end{bmatrix}$ belongs to $\Upsilon_{Y,\max}$ i.e., it is a maximal approximate solution.  
\end{example}

\subsection{Finding approximate solutions
of the rule parameters according to multiple training data}

The equation system $(\Sigma)$ has been introduced for learning the rule parameters according to a training datum \cite{baaj2022learning}. 
Our results let us tackle the problem of determining values of the rule parameters when we have multiple training data  as follows.

\noindent Let us consider that we have $N$ equation systems $(\Sigma_1): Y_1 = \Gamma_1 \Box_{\max}^{\min} X, (\Sigma_2): Y_2 = \Gamma_2 \Box_{\max}^{\min} X, \dots, (\Sigma_N): Y_N = \Gamma_N \Box_{\max}^{\min} X$, where each of them is formed from a training datum using the procedure introduced in \cite{baaj2022learning}.  From the matrices $\Gamma_1, \Gamma_2,\dots, \Gamma_N$ and the second members $Y_1, Y_2, \dots Y_N$   of the equation systems, we form a new matrix and a new column vector by block matrix construction:
\begin{equation}
     \mathbf{\Gamma} = \begin{bmatrix}
        \Gamma_1\\
        \Gamma_2\\
        \vdots\\
        \Gamma_N
    \end{bmatrix} \text{ and } \mathbf{Y} =  \begin{bmatrix}
        Y_1\\
        Y_2\\
        \vdots \\
        Y_N
    \end{bmatrix} 
\end{equation}
\noindent We introduce the following  equation system, which in fact stacks the equation systems $(\Sigma_1), (\Sigma_2),\dots,(\Sigma_N)$ into one:
\begin{equation}
    (\mathbf{\Sigma}): \mathbf{Y} = \mathbf{\Gamma} \Box^{\min}_{\max} X.
\end{equation}
By solving $(\mathbf{\Sigma})$, we obtain solutions for the rule parameters that take into account all the training data. If the system  $(\mathbf{\Sigma})$ is inconsistent, $\nabla(\mathbf{\Gamma},\mathbf{Y})$ is the Chebyshev distance associated to its second member $\mathbf{Y}$  and using our results, we can obtain approximate solutions of the rule parameters which are  approximate solutions of the system $(\mathbf{\Sigma})$.

We illustrate this paradigm by the following example. 
\begin{example}
\noindent We consider two systems, each of them being built from a training datum using the method presented in~\cite{baaj2022learning}:
    \begin{align*}
(\Sigma_1): \, \,\, \, Y_1\, \,\, \, &= \Gamma_1 \Box_{\max}^{\min} X\\
 \begin{bmatrix}
    0.3 \\ 1 \\ 0.3 \\ 0.8 
    \end{bmatrix} &= \begin{bmatrix}
0.4 & 1 & 1 & 1\\ 
1 & 1 & 1 & 1 \\
0.4 & 1 & 1 & 0.8\\
1 & 1 & 1 & 0.8 \\ 
    \end{bmatrix}\Box_{\max}^{\min} \begin{bmatrix}
    s_1 \\
    r_1 \\
    s_2 \\
    r_2 \\
    \end{bmatrix}
    \end{align*}
\noindent and 
    \begin{align*}
    (\Sigma_2): \, \,\, \, Y_2\, \,\, \, &= \Gamma_2 \Box_{\max}^{\min} X\\
  \begin{bmatrix}
    1 \\ 0.8 \\ 0.3 \\ 0.3 
    \end{bmatrix} &= \begin{bmatrix}
1 & 1 & 1 & 1\\ 
1 & 0.7 & 1 & 1 \\
1 & 1 & 1 & 0.1\\
1 & 0.7 & 1 & 0.1 \\ 
    \end{bmatrix}\Box_{\max}^{\min} \begin{bmatrix}
    s_1 \\
    r_1 \\
    s_2 \\
    r_2 \\
    \end{bmatrix}.
\end{align*}
\noindent We remind that $s_1,r_1,s_2,r_2$ are the unknown rule parameters. 
\noindent We form the system $(\mathbf{\Sigma}):$
\begin{align*}
    (\mathbf{\Sigma}): \, \,\, \, \mathbf{Y}\, \,\, \, &= \mathbf{\Gamma} \Box_{\max}^{\min} X\\
   \begin{bmatrix}
    0.3 \\ 1 \\ 0.3 \\ 0.8\\ 1 \\ 0.8 \\ 0.3 \\ 0.3 
    \end{bmatrix} &= \begin{bmatrix}
    0.4 & 1 & 1 & 1\\ 
1 & 1 & 1 & 1 \\
0.4 & 1 & 1 & 0.8\\
1 & 1 & 1 & 0.8 \\ 
1 & 1 & 1 & 1\\ 
1 & 0.7 & 1 & 1 \\
1 & 1 & 1 & 0.1\\
1 & 0.7 & 1 & 0.1 \\ 
    \end{bmatrix}\Box_{\max}^{\min} \begin{bmatrix}
    s_1 \\
    r_1 \\
    s_2 \\
    r_2 \\
    \end{bmatrix}. \\    
\end{align*}
\noindent The system $(\mathbf{\Sigma})$ is inconsistent, because the  Chebyshev distance associated to its second member $\mathbf{Y}$ is $\nabla=0.1$. The lowest Chebyshev approximation of $\mathbf{Y}$ is denoted $ \mathbf{\check{Y}}$ and 
we find that we have a unique maximal  Chebyshev approximation of $\mathbf{Y}$ denoted $\mathbf{\hat{Y}}$:
\begin{equation*}
    \mathbf{\check{Y}} = \begin{bmatrix}
    0.4\\ 1\\0.4\\ 0.8\\ 1\\ 0.7\\ 0.2 \\ 0.2
    \end{bmatrix} \text{ and } \mathbf{\hat{Y}}= \begin{bmatrix}0.4\\ 1\\ 0.4\\ 0.8\\ 1\\ 0.9\\ 0.4 \\  0.4\end{bmatrix}.
\end{equation*}

\noindent Some approximate solutions of the system $(\mathbf{\Sigma}) $ are  the solutions of the system $\check{\mathbf{Y}} = \mathbf{\Gamma} \Box_{\max}^{\min} X$ i.e.,  $
\begin{bmatrix}
0 \\
0.7 \\
0 \\
0.2 \\
 \end{bmatrix}
 \leq 
 \begin{bmatrix}
    s_1 \\
    r_1 \\
    s_2 \\
    r_2 
    \end{bmatrix}
    \leq 
 \begin{bmatrix}
 0.4 \\
 0.7 \\
 1 \\
 0.2
 \end{bmatrix}$ and the  solutions of the system $\hat{\mathbf{Y}} = \mathbf{\Gamma} \Box_{\max}^{\min} X$  i.e, $
\begin{bmatrix}
0 \\
0.9 \\
0 \\
0.4\\
 \end{bmatrix}
 \leq 
 \begin{bmatrix}
    s_1 \\
    r_1 \\
    s_2 \\
    r_2 \\
    \end{bmatrix}
    \leq 
 \begin{bmatrix}
 0.4 \\
 0.9\\
 1 \\
 0.4
 \end{bmatrix}$.
\end{example}

\section{Conclusion}
In this article, for an inconsistent system of  $\max-\min$ fuzzy relational equations denoted $(S): A \Box_{\min}^{\max}x =b$,  we have described the approximate solutions set of the system and the set of Chebyshev approximations of the second member $b$.  The main tool of our study is an explicit analytical formula to compute  the Chebyshev distance $\Delta = \inf_{c \in {\cal C}} \Vert b - c \Vert$, which is expressed in $L_\infty$ norm, and where ${\cal C}$ is the set of second members of the consistent systems defined with the same matrix $A$. The Chebyshev distance is obtained by elementary calculations involving only the components of the matrix $A$ and those of the second member $b$.

We defined an approximate solution of an inconsistent system  $A \Box_{\min}^{\max}x =b$ as a solution of a consistent system $A \Box_{\min}^{\max}x = c$, where $c$ is a vector such that $\Vert b - c \Vert=\Delta$ ; $c$ is called a Chebyshev approximation of $b$. We first related the  approximate solutions set to the set of Chebyshev approximation of $b$. We gave two sharp characterizations of the approximate solutions set and showed how to get minimal Chebyshev approximations of $b$ from minimal approximate solutions. As a consequence of our result, we proved that the set of minimal Chebyshev approximations of $b$ is non-empty and finite. Furthermore, we described the structure of the  approximate solutions set and that of the set of Chebyshev approximations of $b$.

We introduced a paradigm for $\max-\min$ learning approximate weight matrices relating input and output data from training data, where the learning error is expressed in terms of $L_\infty$ norm. For this purpose, we canonically associated to the training data  systems of $\max-\min$ fuzzy relational equations. These systems  allowed us to compute the minimal value $\mu$ of the learning error according to the training data. This minimal value $\mu$ is expressed in terms of the Chebyshev distances associated to the second member of the already introduced systems. Moreover, we gave a method for constructing approximate weight matrices whose learning error is equal to $\mu$.

By introducing analogous tools for a system of $\min-\max$  fuzzy
relational equations to those already introduced for a system of $\max-\min$ fuzzy relational equations, and then establishing the correspondences between them,  we have shown  that the study of the approximate solutions of a system of $\max-\min$ fuzzy relational equations is equivalent to the study of the approximate solutions of a system
of $\min-\max$ fuzzy relational equations. This  allowed us to extend the results of \cite{baaj2022learning}: we gave a method to approximately learn the rule parameters of  a possibilistic rule-based system according to multiple training data. 

In perspectives, we are currently working on the development of analogous tools for  systems of $\max-T$ fuzzy relational equations, where $T$ is the t-norm product or the t-norm of Łukasiewicz. For these systems, we already have analytical formulas to compute the Chebyshev distance associated to their second member. 
As applications, for the problem of the $\max-\min$ invertibility of a fuzzy relation, when a fuzzy matrix $A$ has no preinverse (resp. postinverse), we know how to compute, using the $L_\infty$ norm,  an approximate preinverse  (resp. postinverse) for $A$. We also tackle the development of new applications based on systems of $\max-T$ fuzzy relational equations where $T$ is a t-norm among $\min$, product or the one of Łukasiewicz.

\bibliographystyle{plainnat}

\end{document}